%% file: main.tex
\documentclass{article}

\usepackage[preprint]{neurips_2026}

\usepackage[utf8]{inputenc} % allow utf-8 input
\usepackage[T1]{fontenc}    % use 8-bit T1 fonts
\usepackage{hyperref}       % hyperlinks
\usepackage{url}            % simple URL typesetting
\usepackage{booktabs}       % professional-quality tables
\usepackage{amsfonts}       % blackboard math symbols
\usepackage{nicefrac}       % compact symbols for 1/2, etc.
\usepackage{microtype}      % microtypography
\usepackage{xcolor}         % colors

\input{Defs/mathdefs}
\input{Defs/thmdefs}
\input{Defs/macros}

\usepackage[capitalise]{cleveref}
\usepackage{algorithm}
\usepackage[noend]{algpseudocode}

\algnewcommand{\IfThenElse}[3]{% \IfThenElse{<if>}{<then>}{<else>}
  \State \algorithmicif\ #1\ \algorithmicthen\ #2\ \algorithmicelse\ #3}

\title{Quantile of Means: A Bonus-Free Ensemble Method for Minimax Optimal Reinforcement Learning}

\author{%
  Asaf Cassel \\
  Google Research \\
  \texttt{asafca@google.com} \\
  \And
  Aviv Rosenberg \\
  Google Research \\
  \texttt{avivros@google.com} \\
}

\begin{document}

\maketitle

\begin{abstract}%
    Optimal Reinforcement Learning (RL) algorithms typically rely on carefully constructed count-based uncertainty estimates to drive exploration. Although theoretically sound, such estimates are hard to compute in practical settings and therefore offer limited insight for designing exploration heuristics. Meanwhile, ensembling has emerged as a practical approach, but remains without theoretical justification. Building on a recent ensemble-based method for Multi-Armed Bandits, we propose a quantile-based ensemble method for finite-horizon Markov Decision Processes (MDPs). Our simple count-free approach achieves optimal variance-dependent regret bounds, providing theoretical grounding for ensemble-based exploration in RL.
    % 
    % Optimal reinforcement learning algorithms typically rely on carefully constructed count-based uncertainty estimates to drive exploration. Although theoretically sound, such estimates are hard to compute in practical settings and therefore offer limited insight for designing exploration heuristics. Meanwhile, ensembling has emerged as a practical approach, but remains without theoretical justification. Building on a recent ensemble-based method for Multi-Armed Bandits, we propose a quantile-based extension to reinforcement learning in finite-horizon Markov Decision Processes (MDPs). Our simple count-free approach achieves optimal (variance-dependent) regret bounds, providing a theoretical grounding for ensemble-based exploration.
\end{abstract}

\section{Introduction}
  Reinforcement learning (RL) provides a general framework for sequential decision-making under uncertainty, where an agent interacts with an environment to optimize a long-term objective. RL algorithms have achieved remarkable success across domains ranging from robotics and games \citep{mnih2015human,schulman2015trust,schulman2017proximal,haarnoja2018soft} to Large Language Models (LLMs; \cite{stiennon2020learning,ouyang2022training}). Despite this progress, designing algorithms that can efficiently explore their environment -- gathering informative experience while avoiding excessive trial and error -- remains one of the fundamental challenges in RL.

  Efficient exploration is central to achieving optimal performance in RL. In principle, optimal algorithms rely on carefully constructed uncertainty estimates to balance exploration and exploitation. Count-based approaches, in particular, offer a theoretically elegant
  means of quantifying uncertainty, leading to algorithms with near-optimal regret guarantees. However, these methods depend on explicit access to state-action visitation counts or precise transition models -- quantities that are infeasible to compute or approximate in
  high-dimensional or continuous environments. Consequently, while these methods yield important theoretical insights, they provide limited guidance for practical exploration strategies.

  In contrast, ensemble-based methods have become a staple of modern deep reinforcement learning \citep{lee2021sunrise,chen2021randomized}. By maintaining a collection of function approximators and leveraging their diversity to drive exploration, ensemble methods have achieved empirical success across a wide
  range of domains \citep{osband2016deep,osband2016generalization,pathak2019self,peer2021ensemble}. Yet, despite their practical utility, the theoretical basis for their effectiveness remains poorly understood. In particular, it is unclear how or why ensembling induces the kind of
  optimism or uncertainty quantification required for efficient exploration.

  Recent progress in Multi-Armed Bandits (MAB) has begun to bridge this gap. In \cite{cassel2025batch}, a simple batch ensemble scheme was shown to achieve variance-dependent regret bounds for stochastic bandits -- matching the performance of carefully tuned count-based methods -- without requiring knowledge of the reward distributions or explicit confidence bounds. This result demonstrates that ensemble diversity alone can serve as a proxy for uncertainty, providing a count-free, distribution-agnostic route to optimal exploration in the bandit setting.

  While \cite{viel2025ilsoar} recently adapted the bandit ensemble technique of \cite{cassel2025batch} to MDPs in an imitation learning context, their approach achieves suboptimal rates and requires binarizing the state space -- a construction that scales quadratically with the state space and cannot be applied
  heuristically in a function approximation setting. This leaves open the question of whether a direct ensemble method, without such structural modifications, can achieve optimal regret in MDPs.

  In this work, we answer this question affirmatively. We propose a quantile-based ensemble method for finite-horizon MDPs that selects actions according to a fixed quantile of an ensemble of Q-value estimates. This simple mechanism naturally captures optimism in the face of uncertainty, while requiring no explicit counts, no posterior computation, and no prior knowledge of reward distributions.

  \paragraph{Contributions.}
  \begin{enumerate}
  \item \textbf{Instance-optimal regret for MDPs.} We prove that our algorithm achieves instance-optimal (variance-dependent) regret bounds for tabular finite-horizon MDPs, matching the best known results \citep{zanette19a,zhou2023sharp} -- previously attainable only
  through complex count-based bonus constructions. To our knowledge, this is the first provably efficient ensemble-based exploration algorithm for MDPs, and moreover it achieves optimal rates.

  \item \textbf{Optimal bandit bounds via quantiles.} The quantile-based approach also yields improved results in the bandit setting, shaving a logarithmic factor from the regret of \cite{cassel2025batch} and achieving the optimal instance-dependent rate. See analysis in \cref{sec:MAB}.

  \item \textbf{Distribution-agnostic algorithm, distribution-adaptive analysis.} Our algorithm encodes no distributional assumptions -- it works unchanged for bounded, sub-Gaussian, and non-negative heavy-tailed reward distributions. Yet the analysis can leverage the
  true concentration properties of the underlying distributions (e.g., KL-based bounds for Bernoulli rewards), yielding tighter guarantees without algorithmic modification.

  \item \textbf{Simple and transparent analysis.} We present the complete proof in the main paper. The absence of bonus terms makes the key arguments -- particularly the replacement of empirical value functions with optimal ones for establishing optimism -- considerably
  more transparent than in prior variance-dependent analyses.
  \end{enumerate}

  Extending ensemble-based exploration from bandits to MDPs is substantially more challenging than a per-state reduction: Q-value estimates at each step depend on subsequent value estimates through Bellman backups, breaking the independence
  structure that makes the bandit analysis clean. Standard covering arguments over the value function address this dependence but result in an ensemble size linear in the number of states, rather than logarithmic. Achieving optimal dependence requires
  more careful arguments that replace the empirical value function with the optimal one when establishing optimism -- a technique that, in prior work, is always intertwined with intricate bonus constructions \citep{zanette19a,zhou2023sharp}. In our
  setting, the absence of bonuses makes this argument considerably more transparent, suggesting that ensemble methods may be a more natural algorithmic primitive for exploration in MDPs than the bonus-based paradigm. Whether this approach extends to
  linear or general function approximation settings remains an interesting open question.

  \subsection{Related work.}

  \textbf{Variance-dependent regret in MDPs.} Achieving minimax optimal regret in tabular MDPs has been the focus of a long line of work \citep{azar2017minimax,dann2017unifying,jin2018is}, with tight bounds scaling as $\tilde{O}(\sqrt{H^3 S A T})$. Building on this, \cite{zanette19a} provided the first variance-dependent regret bounds, introducing carefully constructed count-based bonuses that yield a highly complex algorithm and analysis. \cite{zhou2023sharp} recently refined this variance dependence further and established matching lower bounds, confirming the tightness of such problem-dependent results. However, both methods fundamentally require explicit uncertainty quantification through bonuses. In contrast, our algorithm achieves the same variance-dependent guarantees using a simple quantile-based ensemble mechanism—requiring no bonuses, no explicit visitation counts, and no prior knowledge of the reward distribution—while matching the lower bounds of \cite{zhou2023sharp}.

  \textbf{Posterior/quantile-based exploration in MDPs.} Our work also conceptually relates to Posterior Sampling for RL (PSRL) and randomized value functions \citep{osband2013more,agrawal2017optimistic,russo2019worst}. More recently, \cite{tiapkin2022dirichlet} proposed Bayes-UCBVI, which uses quantiles of a posterior over Q-values to achieve optimal $\tilde{O}(\sqrt{H^3SAT})$ regret without explicit bonuses. Similarly, \cite{tiapkin2022opsrl} take an ensembling approach using a few posterior samples. Both algorithms, however, fundamentally rely on maintaining a Dirichlet posterior over transition probabilities—an inherently count-based statistical operation—and assume the reward function is fully known. Adapting these methods to unknown reward distributions requires strong prior knowledge. Our method achieves comparable optimism using a pure data-partitioning ensemble that requires no explicit posterior computation or distributional assumptions.

  \textbf{Ensemble methods in bandits.} Our work directly builds on \cite{cassel2025batch}, which introduced the batch ensemble scheme for stochastic bandits achieving variance-dependent regret without distributional assumptions. This paper is the first to extend this strictly count-free approach to MDPs. Other bandit perturbation and ensemble techniques \citep{lu2017ensemble,KvetonSVWLG19-GIRO,lee2024improved} or bootstrap bounds \citep{hao2019bootstrapping} either fail to adapt to unknown reward distributions, inject artificial noise, or strictly require symmetric rewards. Furthermore, while Bayesian approaches like Bayes-UCB \citep{kaufmann2012bayesian} use posterior quantiles, they still fundamentally rely on computing explicit count-dependent confidence levels.

  \textbf{Practical ensemble methods in deep RL.} 
  Ensembles are a key technique in empirical deep RL exploration. Methods such as Bootstrapped DQN \citep{osband2016deep,osband2016generalization}, UCB Q-ensembles \citep{chen2017ucb}, Ensemble Bootstrapped Q-Learning \citep{peer2021ensemble}, and SUNRISE \citep{lee2021sunrise} have demonstrated remarkable empirical success in complex environments. However, the theoretical justification for why these ensemble heuristics successfully induce exploration has remained largely open. Our work bridges this gap, providing a rigorous theoretical foundation for bonus-free ensemble-driven exploration in RL.

\section{The Quantile of Means (QoM) Estimator}
To build the foundation for our ensemble-based MDP algorithm, we first isolate the core statistical challenge: estimating the mean of a distribution optimistically from a finite sample without relying on explicit uncertainty bonuses. To this end, we introduce the Quantile of Means (QoM) estimator, a generalization of the classic Median of Means (MoM) approach.

Let $X_1, \ldots, X_n$ be i.i.d \emph{nonnegative} random variables, with expectation $\mu$.
% We propose the following mean estimator.
%
First, for any sequence $\hat{\mu}^1, \ldots, \hat{\mu}^B \in \RR$ define its quantile at level $\alpha \in [0,1]$ as
\begin{align}
\label{eq:quantile-def}
    \quantile(\hat{\mu}^b, b \in [B])
    \eqdef
    \hat{\mu}^{(\ceil{\alpha B})}
    ,
\end{align}
where $\hat{\mu}^{(1)}, \ldots, \hat{\mu}^{(B)}$ is the sequence sorted in ascending order, and $\ceil{\cdot}$ is the ceiling function (which rounds numbers towards infinity). Then, given a partition of the $n$ samples into $B$ fixed, disjoint subsets $\mathcal{D}^1, \ldots, \mathcal{D}^B$, the Quantile of Means estimator is defined as:
\begin{align}
\label{eq:quantile-of-means}
    \hat{\mu}_\alpha
    \eqdef
    \quantile \brk*{
    \sum_{X \in \D^b} \frac{X}{\abs{\D^b} + 1}
    ,
    b \in [B]
    }
    .
\end{align}
We establish the following guarantees for the QoM estimator (see proof in \cref{sec:proof-of-qom}).
\begin{lemma}
\label{lemma:quantile-of-means}
    Let $X_1, \ldots, X_n$ be i.i.d. non-negative random variables bounded almost surely by $R$, with expectation $\mu$ and variance $\sigma^2$. Let $\hat{\mu}_\alpha$ be the QoM estimator (\cref{eq:quantile-of-means}) instantiated with $\nb \ge 26 \log \delta^{-1}$ batches and quantile level $\alpha = 1/65$. 
    Then, each of the following bounds holds individually with probability at least $1-\delta$:
    % Let $X$ be a nonnegative random variable, bounded almost surely by $R$, with expectation $\mu$ and variance $\sigma^2$.
    % Let $X_1, \ldots, X_n$ be i.i.d copies of $X$, and let $\hat{\mu}_\alpha$ be the QoM estimator (\cref{eq:quantile-of-means}). If $\alpha = 1/65$, $\nb \ge 26 \log \delta^{-1}$, then with probability at least $1-\delta$ each of the following holds:
    \begin{enumerate}
        \item (\emph{Optimism}) $\hat{\mu}_\alpha \le \mu$.
        \item (\emph{Bias}) If $\abs{\D^b} \ge \floor{n/\nb}$ for all $b \in [\nb]$, then:
        $
            \hat{\mu}_\alpha
            \ge
            \mu
            -
            1.7\sqrt{\frac{ \sigma^2 \nb}{\max\brk[c]{1,n}}}
            -
            \frac{9 R \nb}{\max\brk[c]{1,n}}
            .
        $
    \end{enumerate}
\end{lemma}
Standard estimators achieve \emph{optimism} by explicitly subtracting an uncertainty estimate from an unbiased mean estimator. Consequently, bounding their bias requires carefully designing and computing tight uncertainty bonuses. In contrast, QoM achieves optimism purely through data partitioning and order statistics. This mechanism relies on the following corollary, which states that a minor modification to the standard sample average is optimistic with constant probability.
%In contrast, the optimistic property of QoM relies on the following result, which states that a minor modification to the sample average is optimistic with constant probability.
% is the main technical tool to prove optimism.

\begin{corollary}
\label{corollary:feige}
    Let $X_1, \ldots, X_n$ be i.i.d nonnegative random variables with expectation $\mu$.
    Then for any $c \ge 1/12$ it holds that:
    $
        \Pr\brk*{\sum_{i=1}^{n} \frac{X_i}{n + c} < \mu}
        \ge
        1 / 13
        .
    $
\end{corollary}
\begin{proof}
    Apply \cref{lemma:feige} in \cref{app:existing-results} (\cite{feige2004sums}, Theorem 1) with $X_i / \mu$ and $\delta = c$.
\end{proof}
%
%
% Notice that this is not a concentration inequality and in fact does not require more than the first moment.
%
%
% Using the result, QoM is optimistic if at least $\ceil{\alpha \nb}$ of the batches are optimistic. Because the batches are independent, this is equivalent to a Binomial random variable with $\nb$ trials, and success probability $1/13$ exceeding $\ceil{\alpha \nb}$ successes. We bound this probability using a Chernoff bound to obtain the result.
Crucially, this is not a standard concentration inequality and requires no information beyond the first moment. Using \cref{corollary:feige}, the QoM estimator is optimistic if at least $\lceil \alpha B \rceil$ of the subsets yield optimistic estimates. Because the $B$ subsets are independent, the number of optimistic subsets follows a Binomial distribution with $B$ trials and a success probability of at least $1/13$. Bounding the lower tail of this Binomial variable via a Chernoff bound yields the optimism guarantee in Lemma 1.

The \emph{bias} property of QoM follows a similar logic, replacing \cref{corollary:feige} with Freedman's inequality (\cref{lemma:freedman} in \cref{app:existing-results}). Similar guarantees can be obtained using other tail inequalities. For example, applying Chebyshev's inequality allows us to handle unbounded random variables with finite variance, albeit with significantly worse constants. Importantly, adapting to these different distributions requires essentially no changes to the QoM estimator itself—only to the analysis. This stands in stark contrast to bonus-based estimators, where the algorithmic implementation must change to incorporate distribution-specific bonus terms.
% Crucially, this is achieved essentially without changing the QoM estimator but only its analysis. This is in stark contrast to bonus-based estimators that contain additional bonus dependent terms in the bias. 

Finally, we note that \cite{cassel2025batch} recently considered the Minimum of Means estimator, which is equivalent to QoM with $\alpha = 1/\nb$. They provide similar guarantees in their Lemmas 3 and 4; however, their optimism claim is applicable only for Bernoulli or symmetric random variables, and their bias bound incurs additional logarithmic factors.

\section{Setting and Notations}
\label{sec:setting}

\paragraph{MDP.}
A finite horizon MDP is defined by a tuple $\M = (\X, \A, H, P, \LL)$, where $\X$ is a set of states (of size $S$), $\A$ is a set of actions (of size $A$), $H$ is the decision horizon, $P_h, h \in [H]$ are the transition dynamics, and $\LL$ is the loss distribution.
Each episode $k \in [K]$ starts at state $s_1 \in \X$ and proceeds inductively as follows.
If at step $h \in [H]$ the system is in state $s_h^k \in \X$, and an agent chooses action $a_h^k \in \A$,
then the system transitions to a new state $s_{h+1}^k \in \X$ sampled independently according to $P_h(\cdot \mid s_h^k, a_h^k)$, and a loss $L_h^k \in \RR$ is sampled independently according to $\LL_h(s_h^k, a_h^k)$ whose mean is denoted $\ell_h(s_h^k, a_h^k)$. The process ends after $H$ steps, at which point a new episode begins.
We note that fixing the initial state is done to ease the notation and does not lose generality.

\begin{assumption}[Bounded total loss]
\label{assumption:bounded-total-loss}
The random losses $L_h$ are non-negative and satisfy $\sum_{h \in [H]} L_h \le H$ almost surely for all policies.
\end{assumption}

\paragraph{Policy and Regret.}
We consider agent strategies known as (deterministic) Markov policies $\pi = (\pi_h)_{h \in [H]} : [H] \times \X \mapsto \A$, which map a step and state to an action. Such a policy induces a distribution over trajectories $\iota = (s_h,a_h)_{h \in [H]}$, and we denote expectation with respect to this distribution as $\EE[P, \pi]\brk[s]{\cdot}$.
For each policy $\pi$ and horizon $h \in [H]$ we define its value (or loss to-go) as $V_h^{\pi}(s) = \EE[P,\pi][\sum_{h'=h}^{H} \ell_{h'}(s_{h'},a_{h'}) \mid s_h = s]$, which is the expected loss if one starts from state $s$ at horizon $h$ and follows policy $\pi$.
The performance of a policy, also called its value, is measured by its expected cumulative loss, given by $V_1^{\pi}(s_1)$.
Thus, the optimal policy and value are given by $\pi^\star \in \argmin_{\pi \in \piDMClass} V^{\pi}(s_1)$ and $V^\star = V_1^{\pi^\star}$, where $\piDMClass$ is the class of deterministic Markov policies, which is known to be optimal even among the class of stochastic history-dependent policies.
Finally, we measure the quality of any algorithm via its \emph{regret} -- the difference between the value of the policies $\pi^k$ generated by the algorithm and that of the optimal policy $\pi^\star$, i.e.,
\begin{align*}
    \regret
    =
    \sum_{k \in [K]} V_1^{\pi^k}(s_1) - V_1^\star(s_1)
    .
\end{align*}

\paragraph{Additional notation.}
We use the following notations throughout. $[m]$ for any $m \in \NN$, denotes the set $\brk[c]{1, \ldots, m}$. For any $V: \X \to \RR, s \in \X, a \in \A$ we denote $PV(s,a) = \sum_{s' \in \X} P(s' \mid s,a) V(s')$. For a random variable $X$, let $\Var_{s,a,h}(X)$ be its variance conditioned on $(s_h,a_h) = (s, a)$.

\section{Algorithm and Main Result}
We now introduce Value Iteration with Bootstrap Ensemble (VIBE; \cref{alg:vibe}), an algorithm that replaces explicit uncertainty bonuses with our QoM estimator. 
Standard optimistic algorithms construct explicit upper confidence bounds by estimating the empirical mean and subtracting a bonus proportional to the inverse square root of the visitation count. VIBE bypasses this entirely. Instead, we maintain $B$ independent, non-overlapping datasets $\mathcal{D}^{k,b}_h(s,a)$ for each state-action-step tuple.

Concretely, to ensure these datasets grow at the same rate—thereby maintaining a balanced variance across the ensemble—we assign incoming transitions $(s_{h+1}^k, L_h^k)$ using a strict round-robin schedule:
% Concretely, we divide the next state and loss samples from each $(s,a,h)-$tuple into $\nb$ datasets $\D^{k,b}_h(s,a)$ using a round-robin schedule. Formally
\begin{align}
\label{eq:round-robin}
    \D_h^{k+1,b}(s,a) = \D_h^{k,b}(s,a) \bigcup
    \begin{cases}
        (s_{h+1}^k, L_h^k), &\text{if } (s,a,b) = (s_h^k, a_h^k, b_h^k(s,a))
        \\
        \emptyset, &\text{otherwise},
    \end{cases}
\end{align}
where $b_h^k(s,a) = \argmin_{b \in [\nb]} \abs{\D_h^{k,b}(s,a)}$.

During planning, we perform standard value iteration (without bonuses), but evaluate the $Q-$value using a QoM estimate (see \cref{eq:quantile-of-means}) applied across our $\nb$ (nearly) equal non-overlapping datasets at quantile level $\alpha$. Because we are minimizing loss, taking the $\alpha$-quantile of the $Q$-value estimates naturally filters out overestimations while implicitly establishing a high-probability optimistic lower bound on the true loss-to-go. The full procedure is detailed in Algorithm 1.

From a practical standpoint, VIBE is highly tunable. It requires only two hyperparameters: the quantile level $\alpha$ (which is constant) and the number of batches $B$ (which scales logarithmically with $H, S, K$, and $\delta^{-1}$). Computationally, VIBE essentially runs value iteration $B$ times per episode, taking $O(SAH B)$ time. Because $B$ is logarithmic, this ensemble approach introduces negligible overhead compared to standard bonus-based value iteration.
% Notice that \cref{alg:vibe} has only the number of batches $\nb$ and quantile level $\alpha$ as parameters. As we show, $\alpha$ is constant and $\nb$ is logarithmic in $H,S,K,$ and $\delta^{-1}$, thus making them easy to tune.
% %
% Finally, from a computational perspective, we are essentially running value iteration $\nb$ times in each episode, thus taking $O(SAH \nb)$ time. Because $B$ is logarithmic, this is typically not a significant overhead.

\begin{algorithm}[!ht]
	\caption{Value Iteration with Bootstrap Ensemble (VIBE)} \label{alg:vibe}
	\begin{algorithmic}[1]
	    \State \textbf{input}:
        Number of batches $\nb$, quantile level $\alpha$.

        \State \textbf{initialize}: Datasets $\D_h^{1,b}(s,a) = \emptyset$ for all $s \in S, a \in A, h \in [H]$.
        
	    \For{episode $k = 1,2,\ldots, K$}
            
            \State Define $\hat{V}_{H+1}^{k}(s) = 0$ for all $s \in S$.
            
            \For{$h = H, \ldots, 1$}

                \For{$b \in [\nb]$}
                \State $
                    \hat{P}_h^{k,b}(s' \mid s,a)
                    =
                    \sum_{s_+ \in \D_h^{k,b}(s,a)} \frac{\indEvent{s_+ = s'}}{\abs{\D_h^{k,b}(s,a)} + 1}
                $

                \State $
                    \hat{\ell}_h^{k,b}(s,a)
                    =
                    \sum_{L \in \D_h^{k,b}(s,a)} \frac{L}{\abs{\D_h^{k,b}(s,a)} + 1}
                $

                \State $
                    \hat{Q}_h^{k,b}(s,a)
                    =
                    \brk[s]{\hat{\ell}_h^{k,b} + \hat{P}_h^{k,b} \hat{V}_{h+1}^k}(s,a)
                $

                \EndFor
                
                \State $
                    \hat{V}_h^{k}(s)
                    =
                    \min_{a \in \mathcal{A}} \quantile\brk*{
                    \hat{Q}_h^{k,b}(s,a)
                    ,
                    b \in [\nb]
                    }
                $
                \Comment{QoM \cref{eq:quantile-of-means}}
                % \label{line:vibe-tabular}
                
            \EndFor

            \State Play $\pi^k$ defined as $\pi^k_h(s) \in
                    \argmin_{a \in \mathcal{A}} 
                    \quantile\brk*{
                    \hat{Q}_h^{k,b}(s,a)
                    ,
                    b \in [\nb]
                    }$
                    and observe trajectory $\iota^k$.

            \State Add each $(s_{h+1}^k, L_h^k)$ to $\D^{k+1,b}_h(s_h^k, a_h^k)$ with fewest samples.
            \Comment{Round-Robin \cref{eq:round-robin}}

            % \State Update global counts: $N_h^{k+1}(s_h,a_h) = N_h^{k}(s_h,a_h) + \indEvent{(s_h,a_h) \in \iota^k}$

            % \State Calculate batch index: $b_h^k = 1 + (N_h^{k}(s_h^k, a_h^k) \mod{B})$

            % \State Update batch datasets:
            % \begin{align*}
            %     \D_h^{k+1,b}(s,a) = \D_h^{k,b}(s,a) \bigcup
            %     \begin{cases}
            %         (s_{h+1}^k, L_h^k), &\text{if } (s,a,b) = (s_h^k, a_h^k, b_h^k)
            %         \\
            %         \emptyset, &\text{otherwise}.
            %     \end{cases}
            % \end{align*}
            
	    \EndFor
	\end{algorithmic}
	
\end{algorithm}

\paragraph{Main result.}
To state our theoretical guarantees, we characterize the problem complexity using the maximum cumulative conditional variance of the MDP, similarly to \cite{zanette19a}. We define $\QQ = \sum_{h \in [H]} \QQh$ where
\begin{align}
\label{eq:def-mdp-variance}
    \QQh
    =
    \max_{s \in \X, a \in A} \Var_{s,a,h}(L_h + V_{h+1}^\star(s_{h+1}))
    .
\end{align}
\cref{thm:vibe} establishes that VIBE achieves a variance-dependent regret bound without requiring explicit variance estimation or bonuses. See proof in \cref{sec:proof-vibe}, and discussion on extending to heavy-tailed losses in \cref{sec:heavy-tail-extension}.
\begin{theorem}
\label{thm:vibe}
    Suppose we run \cref{alg:vibe} with $\alpha = 1/65$ and $\nb = 26 \log (5 S A H K \delta^{-1})$, and let $\kappa = \log (20 H S^2 A K \delta^{-1})$ be a logarithmic factor. With probability at least $1-\delta$,
    \begin{align*}
        \regret
        \le
        22 \sqrt{\min\brk[c]{\QQ, H V^\star} H S A K \kappa^2}
        +
        1924 H^3 S^2 A \kappa^3
        .
    \end{align*}
\end{theorem}
% Crucially, the leading term of this bound matches the $\Omega(\sqrt{H S A T})$ minimax lower bound for episodic tabular MDPs (e.g., \cite{domingues21a}) up to logarithmic factors, while adapting strictly to the problem-dependent variance $\mathcal{Q}^*$ when the environment is near-deterministic.
Crucially, the leading term of this bound establishes the optimality of our ensemble approach in two distinct regimes (up to logarithmic factors). In the worst case, it matches the standard $\Omega(\sqrt{H^3 S A K})$ minimax lower bound of \citep{domingues21a}. Moreover, in environments with low stochasticity, the bound scales with the problem-dependent variance $\mathcal{Q}^*$, matching the refined variance-dependent lower bounds established by \cite{zhou2023sharp}. 
% Crucially, the leading term of this bound establishes the optimality of our ensemble approach in two distinct regimes. In the worst case, it matches the standard $\Omega(\sqrt{H^3 S A K})$ minimax lower bound for episodic tabular MDPs \citep{domingues21a} up to logarithmic factors. Moreover, in environments with low stochasticity, the bound scales strictly with the problem-dependent variance $\mathcal{Q}^*$, matching the refined variance-dependent lower bounds established by \cite{zhou2023sharp}.

\section{Ensembles as a Tool for Optimism}
\label{sec:optimism}
Having established VIBE's theoretical guarantees and algorithmic simplicity, we now dissect the mathematical engine driving these results. The purpose of this section is to formally demonstrate how our ensemble quantile mechanism naturally induces optimism in the MDP setting, bypassing the need for explicit uncertainty bonuses.
To see this, we first decompose the cumulative regret into two standard components—colloquially referred to as bias and optimism:
\begin{align}
\label{eq:regret-decomp}
    \regret
    =
    \sum_{k \in [K]} 
    \underbrace{
    V_1^{\pi^k}(s_1) - \hat{V}_1^k(s_1)
    }_{\text{bias}}
    +
    \sum_{k \in [K]} 
    \underbrace{
    \hat{V}_1^k(s_1) - V_1^\star(s_1)
    }_{\text{optimism}}
    .
\end{align}
The primary goal of optimistic algorithms is to ensure that the $\text{optimism}$ term is non-positive , and thus the total regret is bounded entirely by the bias term. Traditionally, bounding the optimism term requires carefully designed count-based bonuses, which then become the dominant factor when analyzing the bias. Because VIBE completely forgoes explicit bonuses, the analysis of our bias term is significantly simplified and relies mostly on standard techniques (detailed in \cref{sec:proof-mdp-bias}).
% The former is usually achieved using carefully designed, count-based bonuses, which also become the dominant term in the bias.
%
% As we forgo the use of bonuses, the analysis of the bias term (\cref{sec:proof-mdp-bias}), which uses mostly standard techniques, is also significantly simplified.
In what follows, we show that our QoM approach achieves optimism on the following high-probability event.

% Building on this result, we show that our bootstrap ensemble estimator is optimistic with high probability.
% \begin{lemma}
% \label{lemma:bootstrap-ensemble}
%     Let $X_1, \ldots, X_n$ be i.i.d nonnegative random variables, with expectations $\mu$. Let $\D^b, b \in [\nb]$ be fixed disjoint subsets of $\brk[c]{X_1, \ldots, X_n}$.
%     Then for $\nb \ge (25/2) \log \delta^{-1}$ and $c \ge 1/12$, with probability at least $1-\delta$
%     \begin{align*}
%         \min_{b \in [\nb]} \sum_{X \in \D^b} \frac{X}{\abs{\D^b} + c}
%         \le
%         \mu
%         .
%     \end{align*}
% \end{lemma}
% \begin{proof}
%     Because the subsets $\D^b$ are fixed and disjoint, they are independent, thus
%     \begin{align*}
%         \Pr
%         % &
%         \brk*{
%         \min_{b \in [\nb]}
%          \sum_{X \in \D^b} \frac{X}{\abs{\D^b} + c}
%         \le
%         \mu
%         }
%         % \\
%         &
%         =
%         1
%         -
%         \prod_{b \in [\nb]}
%         \Pr
%         \brk*{
%          \sum_{X \in \D^b} \frac{X}{\abs{\D^b} + c}
%         >
%         \mu
%         }
%         \\
%         \tag{\cref{corollary:feige}}
%         &
%         \ge
%         1 - (12/13)^b
%         \\
%         &
%         \ge
%         1 - \delta
%         .
%         \qedhere
%     \end{align*}
% \end{proof}

% The following Lemma incorporates \cref{lemma:quantile-of-means} into part of the good event that implies optimism.
\begin{lemma}[Good event for optimism]
\label{lemma:good-optimism}
    Suppose that $\alpha = 1/65$ and $\nb \ge 26\log (5 S H K \delta^{-1})$, then with probability at least $1 - \delta/5$ simultaneously for all $h \in [H], s \in \X, k \in [K]$.
    \begin{align*}
        \quantile\brk*{
        \brk[s]{\hat{\ell}_h^{k,b} + \hat{P}_h^{k,b} V_{h+1}^\star}(s, \pi_h^\star(s))
        ,
        b \in [\nb]
        }
        \le
        \brk[s]{\ell_h + P_h V^\star_{h+1}}(s, \pi_h^\star(s))
        .
    \end{align*}
    % \begin{align*}
    %     \min_{b \in [\nb]} \frac{1}{\abs{\D_h^{k,b}(s,a^\star(s))} + 1} \sum_{\ell, s' \in \D_h^{k,b}(s,a^\star(s))} \brk[s]{\ell + V^\star_{h+1}(s')}
    %     \le
    %     \ell_h(s, a^\star(s)) + P_h V^\star_{h+1}(s, a^\star(s))
    %     .
    % \end{align*}
\end{lemma}

\begin{proof} 
    First, plugging in the definitions of $\hat{\ell}$ and $\hat{P}$ (see \cref{alg:vibe}), we have
    $$
    \brk[s]{
    \hat{\ell}_h^{k,b} + \hat{P}_h^{k,b} V_{h+1}^\star
    }(s,a)
    =
    \sum_{L, s' \in \D_h^{k,b}(s,a)} \frac{L + V^\star_{h+1}(s')}{\abs{\D_h^{k,b}(s,a)} + 1}
    .
    $$
    Moreover, the datasets $\D_h^{k,b}(s,a)$ are disjoint and each contain i.i.d. samples that satisfy $\EE\brk[s]{L + V_{h+1}^\star(s')} = \brk[s]{\ell_h + P_h V_{h+1}^\star}(s,a)$. Thus, the proof is concluded by performing a union bound on \cref{lemma:quantile-of-means} with respect to $h \in [H], s \in \X$, and all valid configurations of dataset sizes $\abs{\D_h^{k,b}(s, \pi^\star(s))}, b \in [\nb]$. Because we use a round-robin schedule\footnote{There is a one-to-one mapping from the total number of samples, which is at most $K$, to the current configuration.}, there are at most $K$ valid configurations, and thus the total number of events is bounded by $S H K$, as desired.
\end{proof}

\begin{lemma}[Optimism]
\label{lemma:mdp-optimism}
    Conditioning on the good event defined in \cref{lemma:good-optimism}, we have $\hat{V}^k_h(s) \le V_h^\star(s)$ for all $s \in \X, h \in [H], k \in [K]$.
\end{lemma}
\begin{proof}
    We prove the claim by reverse induction on $h \in [H+1]$. The base case is satisfied because $\hat{V}_{H+1}^k = V_{H+1}^\star = 0$. Now, assuming the hypothesis holds for $h+1$, we have
    \begin{align*}
        \hat{V}^k_h(s)
        &
        =
        \min_{a \in \A} \quantile\brk*{
        \hat{Q}_h^{k,b}(s,a)
        ,
        b \in [\nb]
        }
        \\
        &
        \le
        \quantile\brk*{
        \hat{Q}_h^{k,b}(s,\pi^\star(s))
        ,
        b \in [\nb]
        }
        \\
        &
        =
        \quantile\brk*{
        \brk[s]{\hat{\ell}_h^{k,b} + \hat{P}_h^{k,b} \hat{V}_{h+1}^k}(s, \pi_h^\star(s))
        ,
        b \in [\nb]
        }
        \\
        \tag{induction hypothesis}
        &
        \le
        \quantile\brk*{
        \brk[s]{\hat{\ell}_h^{k,b} + \hat{P}_h^{k,b} V_{h+1}^\star}(s, \pi_h^\star(s))
        ,
        b \in [\nb]
        }
        \\
        \tag{Good event - \cref{lemma:good-optimism}}
        &
        \le
        \brk[s]{\ell_h + P_h V^\star_{h+1}}(s, \pi_h^\star(s))
        \\
        \tag{Bellman optimality equation}
        &
        =
        V^\star_h(s)
        ,
    \end{align*}
    where the second and third inequalities also used that $\quantile(x) \le \quantile(y)$ for $x \le y$ (point-wise).
\end{proof}

\section{Regret Analysis}
\label{sec:proof-vibe}

As established in \cref{sec:optimism}, our ensemble approach ensures the optimism term in the regret decomposition is non-positive. Consequently, the regret of VIBE is bounded entirely by the bias.
To formalize this, we first define our empirical counters and operators.
Let $N_h^k(s,a) = \sum_{k' \in [k-1]} \indEvent{(s,a) = (s_h^{k'}, a_h^{k'})}$ be the total visitation count. Because the data is distributed across $B$ batches, the effective inverse count per batch is given by $\beta_h^k(s,a)=\nb/{\max\brk[c]{1, N_h^k(s,a)}}$.
Let $\hat{\mathcal{T}}^{k}_h$ denote the empirical QoM Bellman operator:
\begin{align*}
    \hat{\mathcal{T}}^{k}_h V(s,a)
    =
    \quantile \brk*{\brk[s]{
    \hat{\ell}_h^{k,b} + \hat{P}_h^{k,b} V
    }(s, a)
    ,
    b \in [\nb]
    }
    .
\end{align*}

% We show that, conditioned on the following so-called ``good'' event, the regret of \cref{alg:vibe} is deterministically bounded. 
% %
% To ease notations, we define
% \begin{gather}
% \label{eq:def-beta}
% \begin{aligned}
%     N_h^k(s,a) = \sum_{k' \in [k-1]} \indEvent{(s,a) = (s_h^{k'}, a_h^{k'})}
%     ,
%     \quad
%     \beta_h^k(s,a)
%     =
%     \nb/{\max\brk[c]{1, N_h^k(s,a)}}
%     ,
%     % \quad
%     \\
%     \kappa 
%     =
%     \log 20 H S^2 A K \delta^{-1}
%     ,
%     \quad
%     \hat{\mathcal{T}}^{k}_h V(s,a)
%     =
%     \quantile \brk*{\brk[s]{
%     \hat{\ell}_h^{k,b} + \hat{P}_h^{k,b} V
%     }(s, a)
%     ,
%     b \in [\nb]
%     }
%     ,
% \end{aligned}
% \end{gather}
% where $N_h^k(s,a)$ is the total number of samples associated with the tuple $(s,a,h)$ up to (not including) episode $k$, and $\hat{\mathcal{T}}^{k}_h$ is a QoM estimate of the Bellman operator.
%
%
We define the good event $\goodEvent$ as the intersection of several high-probability inequalities that hold simultaneously for all $s,s' \in \mathcal{S}, a \in \mathcal{A}, h \in [H], k \in [K]$, and $b \in [B]$:
\begin{alignat}{2}
    \label{eq:good-optimism-mdp}
    &
    \hat{\mathcal{T}}^{k}_h V_{h+1}^\star(s,\pi_h^\star(s))
    % \quantile \brk*{\brk[s]{
    % \hat{\ell}_h^{k,b} + \hat{P}_h^{k,b} V_{h+1}^\star
    % }(s, \pi_h^\star(s))
    % ,
    % b \in [\nb]
    % }
    &&
    \le
    \brk[s]{\ell_h + P_h V_{h+1}^\star}(s, \pi_h^\star(s))
    \\
    \label{eq:good-bias}
    &
    \hat{\mathcal{T}}^{k}_h V_{h+1}^\star(s, a)
    % \quantile\brk*{
    % \brk[s]{(\ell_h - \hat{\ell}_h^{k,b}) + (P_h - \hat{P}_h^{k,b}) V_{h+1}^{\star}
    % }(s,a)
    % ,
    % b \in \nb
    % }
    &&
    \ge
    \brk[s]{\ell_h + P_h V_{h+1}^\star}(s, a)
    -
    1.7\sqrt{\beta_h^k(s,a) \Var_{s,a,h}(L_h + V_{h+1}^\star(s_{h+1}))}
    -
    9 H \beta_h^k(s,a)
    % \\
    % &
    % \hat{\mathcal{T}}^{k}_h V_{h+1}^\star(s, a)
    % &&
    % \ge
    % \brk[s]{\ell_h + P_h V_{h+1}^\star}(s, a)
    % -
    % 0.85 \eta^{-1} 
    % \Var_{s,a}(L_h + V_{h+1}^\star(s'))
    % -
    % (0.85 \eta + 9 H) \beta_h^k(s,a)
    \\
    \label{eq:good-phat1}
    &
    \hat{P}_h^{k,b}(s' \mid s,a)
    &&
    \le
    \brk*{1 + \frac{1}{3H}} P_h(s' \mid s,a)
    +
    % 3(e-2)
    2.2 H \kappa \beta_h^k(s,a)
    \\
    \label{eq:good-phat2}
    &
    \hat{P}_h^{k,b}(s' \mid s,a)
    &&
    \ge
    \brk*{1 - \frac{1}{3H}} P_h(s' \mid s,a)
    -
    % 3(e-2)
    2.2 H \kappa \beta_h^k(s,a)
    \\
    \label{eq:good-beta}
    &
    \sum_{k \in [K]} \EE[\pi^k, P] \sum_{h \in [H]} &&\beta_h^k(s_h,a_h)
    % &&
    \le
    % 2 \sum_{k \in [K]} \sum_{h \in [H]} \beta_h^k(s_h^k, a_h^k)
    % +
    2 \nb H S A \kappa
    .
\end{alignat}
These conditions capture the essential statistical properties needed to bound the regret. Equation~\eqref{eq:good-optimism-mdp} establishes optimism, ensuring the QoM estimator does not overestimate the loss of the optimal action. Equation~\eqref{eq:good-bias} bounds the algorithm's bias by limiting how much the estimator can underestimate the true Bellman target for any action. Equations~\eqref{eq:good-phat1} and \eqref{eq:good-phat2} ensure the empirical transition probabilities concentrate around the true dynamics. Notice that the error here scales multiplicatively with the true transition probabilities, which is a strict requirement for achieving variance-dependent bounds. Finally, Equation~\eqref{eq:good-beta} bounds the cumulative sum of the statistical errors (the inverse counts) encountered by the agent across all episodes.

The following two lemmas establish that $\mathcal{E}$ holds with high probability, and that under $\mathcal{E}$, the cumulative bias is bounded by the problem's structural variance (proofs in \cref{sec:proof-mdp-bias,sec:proof-good-event}).
\begin{lemma}[Good Event]
\label{lemma:good-event}
    If $\alpha = 1/65$ and $\nb \ge 26 \log (5 S A H K \delta^{-1})$ then
    $\Pr(\goodEvent) \ge 1 - \delta.$
\end{lemma}
% The following Lemma bounds the bias term (see \cref{eq:regret-decomp}). Although it is the most technical part of the proof, the lack of bonuses simplifies it significantly compared to past results. See proof in \cref{sec:proof-mdp-bias}.
\begin{lemma}[Bias]
\label{lemma:mdp-bias}
    Suppose that the good event (\crefrange{eq:good-optimism-mdp}{eq:good-beta}) holds, then for any $\eta > 0$
    \begin{align*}
        V^{\pi^k}&(s_1) - \hat{V}^k(s_1)
        % \\
        % &
        \le
        \frac{\min\brk[c]{\QQ, H V^\star}}{\eta}
        +
        (8.5\eta + 37 H^2 S \kappa)
        \EE[P, \pi^k]\brk[s]*{
        \sum_{h \in [H]} \beta_h^k(s_h, a_h)
        }
        .
    \end{align*}
\end{lemma}
With these tools, the proof of our main result (\cref{thm:vibe}) follows by conditioning on the good event, bounding the optimism term by $0$, summing the bias term over all $K$ episodes, using \cref{eq:good-beta} and optimizing the free parameter $\eta=\sqrt{ \min\brk[c]{\QQ, H V^\star} K / 17 \nb H S A \kappa}$.
%
%
% Now, assume that the good event holds, which occurs with probability at least $1-\delta$ (\cref{lemma:good-event}). Then, combining \cref{lemma:mdp-optimism,lemma:mdp-bias} gives that for any $\eta > 0$
% %
% \begin{align*}
%     \regret
%     % &
%     % =
%     % \sum_{k \in [K]} V^{\pi^k}(s_1) - V^\star(s_1)
%     % \\
%     &
%     \le
%     \eta^{-1} \min\brk[c]{\QQ, H V^\star} K
%     +
%     (8.5\eta + 37 H^2 S \kappa)
%     \sum_{k \in [K]}
%     \EE[P, \pi^k]\brk[s]*{
%     \sum_{h \in [H]} \beta_h^k(s_h, a_h)
%     }
%     .
% \end{align*}
% %
% %
% Using \cref{eq:good-beta} of the good event, we then get
% %
% \begin{align*}
%     \regret
%     % \tag{\cref{eq:good-beta}}
%     &
%     \le
%     \eta^{-1} \min\brk[c]{\QQ, H V^\star} K
%     +
%     (17\eta + 74 H^2 S \kappa)
%     \nb H S A \kappa
%     .
% \end{align*}
% % where the second inequality used \cref{lemma:counts-sum}, a standard algebraic argument on the harmonic sum of the state-action counts.
% % \begin{align*}
% %     \regret
% %     &
% %     =
% %     24 \sqrt{H^3 S A K \beta^1 \log 16 K \delta^{-1}}
% %     +
% %     6 H^3 S^2 A
% %     \beta^1 \log 16 K \delta^{-1}
% %     \\
% %     &
% %     \le
% %     85 \sqrt{H^3 S A K \log^3 16 H S^2 A K \delta^{-1}}
% %     +
% %     75 H^3 S^2 A
% %     \log^3 16 H S^2 A K \delta^{-1}
% %     ,
% % \end{align*}
% Finally, choosing $\eta = \sqrt{ \min\brk[c]{\QQ, H V^\star} K / 17 \nb H S A \kappa}$, and plugging in the value of $\nb$ and $\kappa$ (see \cref{eq:def-beta}) concludes the proof. \hfill$\blacksquare$

\subsection{Proof of Lemma \ref{lemma:mdp-bias} (Bias Bound)}
\label{sec:proof-mdp-bias}

To ease notation, let
$
\hat{V}^{k,b}_h(s) = \brk[s]{
    \hat{\ell}_h^{k,b} + \hat{P}_h^{k,b} \hat{V}_{h+1}^{k}
}(s, \pi_h^k(s))
$
be the predicted value of batch $b \in [\nb]$ at step $h \in [H]$ of episode $k \in [K]$.
Then, for any $s \in \X, h \in [H], k \in [K], b \in [\nb]$, we decompose the error using the Bellman equations:
\begin{align}
\label{eq:bias-proof-1}
    &
    V_h^{\pi^k}(s) - \hat{V}_h^{k,b}(s)
    % &
    =
    \brk[s]{
    (\ell_h - \hat{\ell}_h^{k,b}) + P_h V_{h+1}^{\pi^k} - \hat{P}_h^{k,b} \hat{V}_{h+1}^k
    }(s, \pi_h^k(s))
    \\
    &
    =
    \brk[s]{
    (\ell_h - \hat{\ell}_h^{k,b})
    +
    (P_h - \hat{P}_h^{k,b}) V_{h+1}^{\star}
    +
    (P_h - \hat{P}_h^{k,b}) (V_{h+1}^{\pi^k} - V_{h+1}^\star)
    +
    \hat{P}_h^{k,b} (V_{h+1}^{\pi^k} - \hat{V}_{h+1}^k)
    }(s, \pi_h^k(s))
    .
    \nonumber
\end{align}
We bound the terms involving $V_{h+1}^{\pi^k}$ using our transition bounds. Using optimism (\cref{lemma:mdp-optimism}) and bounded loss Assumption~\ref{assumption:bounded-total-loss}, we have
$
0
\le
V_{h+1}^{\pi^k} - \hat{V}_{h+1}^k
\le
H
.
$
Thus, the good event (\cref{eq:good-phat1}) gives:
\begin{align*}
    \hat{P}_h^{k,b} (V_{h+1}^{\pi^k} - \hat{V}_{h+1}^k)
    \le
    \brk*{1 + \frac{1}{3H}} P_h (V_{h+1}^{\pi^k} - \hat{V}_{h+1}^k) 
    +
    2.2 H^2 S \kappa \beta_h^k
    .
\end{align*}
Similarly,
$
0
\le
V_{h+1}^{\pi^k} - V_{h+1}^{\star}
\le
H
,
$
so applying \cref{eq:good-phat2} and the optimism property yields:
\begin{align*}
    (P_h - \hat{P}_h^{k,b}) (V_{h+1}^{\pi^k} - V_{h+1}^\star)
    &
    \le
    \frac{1}{3H} P_h (V_{h+1}^{\pi^k} - V_{h+1}^\star) + 2.2 H^2 S \kappa \beta_h^k
    \\
    &
    \le
    \frac{1}{3H} P_h (V_{h+1}^{\pi^k} - \hat{V}_{h+1}^k) + 2.2 H^2 S \kappa \beta_h^k
    .
\end{align*}
Plugging the last two inequalities into \cref{eq:bias-proof-1} gives 
\begin{align*}
    V_h^{\pi^k}(s) - \hat{V}_h^{k,b}(s)
    &
    \le
    \brk[s]{
    (\ell_h - \hat{\ell}_h^{k,b})
    +
    (P_h - \hat{P}_h^{k,b}) V_{h+1}^{\star} }(s, \pi_h^k(s))
    \\
    &
    \quad +
    \brk[s]{4.4 H^2 S \kappa \beta_h^k
    +
    (1 + 2/(3H)) P_h (V_{h+1}^{\pi^k} - \hat{V}_{h+1}^k)
    }(s, \pi_h^k(s))
    % \\
    % &
    % +
    % 4.4 H^2 S \kappa \beta_h^k(s, \pi^k(s))
    .
\end{align*}
% 
% 
% Now, recall that conditioned on $s,a,h$, the random variables $L_h$ and $s'$ are independent, thus
% $$
% \Var_{s,\pi^k(s)}(L_h + V_{h+1}^\star(s'))
% =
% \Var_{s,\pi^k(s)}(L_h)
% +
% \Var_{s,\pi^k(s)}(V_{h+1}^\star(s'))
% .
% $$
Now, notice that $\hat{V}_h^{k} = \quantile(\hat{V}_h^{k,b}, b \in [\nb])$ (see \cref{alg:vibe}) and that the quantile (\cref{eq:quantile-def}) is equivariant to shifts, i.e., $\quantile(\mu_b + c, b \in [\nb]) = \quantile(\mu_b, b \in [\nb]) + c.$
Applying this with the good event's \cref{eq:good-bias}, and noticing that $\kappa \ge \log 40$ we get the step-wise recursion:
\begin{align}
\label{eq:bias-proof-2}
\begin{aligned}
    V_h^{\pi^k}(s) - \hat{V}_h^k(s)
    &
    \le
    1.7\sqrt{\beta_h^k(s, \pi_h^k(s)) \Var_{s,\pi_h^k(s),h}(L_h + V_{h+1}^\star(s_{h+1}))}
    +
    7 H^2 S \kappa \beta_h^k(s, \pi^k(s))
    \\
    &
    \quad +
    (1 + 2/(3H)) P_h (V_{h+1}^{\pi^k} - \hat{V}_{h+1}^k)
    (s, \pi_h^k(s))
    .
\end{aligned}
\end{align}
Recalling that
$
\Var_{s,\pi_h^k(s),h}(L_h + V_{h+1}^\star(s_{h+1})) \le \QQh
$
(see \cref{eq:def-mdp-variance}),
and unrolling the recursion, we get that for any $\eta > 0$
\begin{align*}
    % &
    V_1^{\pi^k}(s_1) - \hat{V}_1^k(s_1)
    % \\
    &
    \le
    \EE[P, \pi^k]\sum_{h \in [H]} (1+1/H)^{h-1} \brk[s]*{
    1.7\sqrt{\QQh \beta_h^k(s_h, a_h)}
    +
    7 H^2 S \kappa \beta_h^k(s_h, a_h)
    }
    \\
    &
    \le
    \eta^{-1} \QQ
    +
    (5.5\eta + 20 H^2 S \kappa) \EE[P, \pi^k]\sum_{h \in [H]} \beta_h^k(s_h, a_h)
    ,
\end{align*}
where the second inequality used $(1+1/H)^h \le e$ for all $h \in [H]$, and the AM-GM inequality ($2\sqrt{ab} \le a + b$). This completes the first part of the lemma.

Next, by optimism (\cref{lemma:mdp-optimism}) and the definition of $V^\star$, we have $\hat{V}^k_{h+1} \le V^\star_{h+1} \le V^{\pi^k}$. Thus, plugging \cref{lemma:var-star-to-pi} into \cref{eq:bias-proof-2}, we get that for any $\eta > 0$
\begin{align*}
    V_h^{\pi^k}(s) - \hat{V}_h^k(s)
    &
    \le
    \eta^{-1} \Var_{s,\pi_h^k(s),h}(L_h + V_{h+1}^{\pi^k}(s_{h+1}))
    +
    ((2/e) \eta + 8 H^2 S \kappa) \beta_h^k(s, \pi^k(s))
    \\
    &
    \quad +
    (1 + 1/H) P_h (V_{h+1}^{\pi^k} - \hat{V}_{h+1}^k)
    (s, \pi^k(s))
    .
\end{align*}
Choosing $\eta \ge 5 e H^2$ and
similarly unrolling the recursion, we have
\begin{align*}
    &
    V_1^{\pi^k}(s_1) - \hat{V}_1^k(s_1)
    \\
    &
    \le
    \EE[P, \pi^k]\sum_{h \in [H]} (1+1/H)^{h-1} \brk[s]*{
    \eta^{-1} \Var_{s,\pi_h^k(s),h}(L_h + V_{h+1}^{\pi^k}(s_{h+1}))
    +
    ((2/e) \eta + 8 H^2 S \kappa) \beta_h^k(s, \pi^k(s))
    }
    \\
    \tag{\cref{lemma:variance-bound}}
    &
    \le
    e \eta^{-1} H V^{\pi^k}_1(s_1)
    +
    (2\eta + 22 H^2 S \kappa) \EE[P, \pi^k]\sum_{h \in [H]} \beta_h^k(s_h, a_h)
    \\
    &
    \le
    (V^{\pi^k}_1(s_1) - V^\star)/5
    +
    e \eta^{-1} H V^{\star}
    +
    (2\eta + 22 H^2 S \kappa) \EE[P, \pi^k]\sum_{h \in [H]} \beta_h^k(s_h, a_h)
    ,
\end{align*}
where \cref{lemma:variance-bound} is an application of the Law of Total Variance \citep{azar2017minimax} together with the bounded total loss assumption. 
Using that $\hat{V}_1^k(s_1) \le V^\star$ (\cref{lemma:mdp-optimism}) and rearranging, we have 
\begin{align*}
    &
    V_1^{\pi^k}(s_1) - \hat{V}_1^k(s_1)
    \le
    3.4 \eta^{-1} H V^{\star}
    +
    (2.5\eta + 27.5 H^2 S \kappa) \EE[P, \pi^k]\sum_{h \in [H]} \beta_h^k(s_h, a_h)
    .
\end{align*}
Replacing $\eta$ with $3.4\eta + 5eH^2$ and using $\kappa \ge \log 40 \approx 3.68$ concludes the proof.
\hfill$\blacksquare$

\subsection{Proof of Lemma \ref{lemma:good-event} (Good Event)}
\label{sec:proof-good-event}

Because of the round-robin schedule (\cref{eq:round-robin}), the samples from each state-action pair are (nearly) uniformly distributed between the datasets. Thus, we can relate the size of each dataset with the global counts as $\abs{D_h^{k,b}(s,a)} \ge \floor{N_h^k(s,a) / \nb}$. Therefore, for all $s \in \X, a \in \A, h \in [H], k \in [K],$ and $b \in [\nb]$ we have
\begin{align}
\label{eq:counts-bound}
    \abs{D_h^{k,b}(s,a)} + 1
    \ge
    \nb^{-1} \max\brk[c]{1, N_h^k(s,a)}
    .
\end{align}
% 
% \begin{proof}
%     The samples of each state-action pair are distributed evenly between batches, and thus
%     \begin{align*}
%         \abs{D_h^{k,b}(s,a)} + 1
%         \ge
%         \floor{N_h^k(s,a) / \nb} + 1
%         \ge
%         \nb^{-1} N_h^k(s,a)
%         .
%     \end{align*}
%     Since $\nb \ge 1$ we also have
%     $
%     \abs{D_h^{k,b}(s,a)} + 1
%     \ge
%     1
%     \ge \nb^{-1}
%     ,
%     $
%     concluding the proof.
% \end{proof}
%
%
%
%
%
Now, we show that each of \crefrange{eq:good-optimism-mdp}{eq:good-beta} hold with probability at least $1-\delta/5$; thus, taking a union bound concludes the proof of \cref{lemma:good-event}.

\textbf{(\cref{eq:good-optimism-mdp}):} Follows from \cref{lemma:good-optimism}.

\textbf{(\cref{eq:good-bias}):}
As shown in \cref{lemma:good-optimism}, $\hat{\mathcal{T}}^{k}_h V_{h+1}^\star(s, a)$ satisfies the basic conditions of \cref{lemma:quantile-of-means}. To satisfy the additional conditions in the second claim of \cref{lemma:quantile-of-means}, we use the bounded total loss assumption to bound each sample by $H$, and \cref{eq:counts-bound} to bound the size of the datasets. Thus, the proof is concluded via a union bound on $s \in \X, a \in \A, h \in [H]$, and the (at most $K$) dataset configurations\footnote{See explanation in the proof of \cref{lemma:good-optimism}}.

\textbf{(\cref{eq:good-phat1,eq:good-phat2}):}
Fix
$s', s \in \X, a \in \A, h \in [H], b \in [\nb], \abs{\D_h^{k,b}(s,a)} \in [K/\nb]$.
For the first inequality,
we apply \cref{lemma:freedman} with $X_t = \indEvent{s_+ = s'} - p_h(s' \mid s,a)$, for $s_+ \in \D_h^{k,b}(s,a)$ and $\lambda = 1/3(e-2)H$. Notice that $R = 1$ and
\begin{align*}
    \EE\brk[s]{X_t^2 \mid X_1, \ldots, X_{t-1}}
    =
    P_h(s' \mid s,a) (1 - P_h(s' \mid s,a))
    \le
    P_h(s' \mid s,a)
    ,
\end{align*}
where the last inequality holds since $P_h \in [0,1]$. Thus, by \cref{lemma:freedman}, with probability at least $1-\delta$
\begin{align*}
    \sum_{s_+ \in \D_h^{k,b}(s,a)} \indEvent{s_+ = s'} - P_h(s' \mid s,a)
    &
    \le
    \frac{1}{3H} \abs{\D_h^{k,b}(s,a)} p_h(s' \mid s,a)
    +
    3(e-2) H \log \frac{1}{\delta}
    .
\end{align*}
Recalling the definition of $\hat{P}$ in \cref{alg:vibe}, 
rearranging the inequality and dividing by $\abs{\D_h^{k,b}(s,a)} + 1$ gives
\begin{align}
\label{eq:phat-freedman1}
    \hat{P}_h^{k,b}(s' \mid s,a)
    -
    \brk*{1 + \frac{1}{3H}} P_h(s' \mid s,a)
    \le 
    \frac{3(e-2)H \log \delta^{-1}}{\abs{\D_h^{k,b}(s,a)} + 1}
    \le
    \frac{3(e-2)H \nb \log \delta^{-1}}{\max\brk[c]{1,N_h^k(s,a)}}
    ,
\end{align}
where the last inequality used \cref{eq:counts-bound}.
For the second inequality, we apply \cref{lemma:freedman} with $X_t = P_h(s' \mid s,a) - \indEvent{s_+ = s'}$ to get that with probability at least $1-\delta$
\begin{align*}
    \sum_{s_+ \in \D_h^{k,b}(s,a)} P_h(s' \mid s,a) - \indEvent{s_+ = s'}
    &
    \le
    \frac{1}{3H} \abs{\D_h^{k,b}(s,a)} p_h(s' \mid s,a)
    +
    3(e-2)H \log \frac{1}{\delta}
    .
\end{align*}
Adding $P_h(s' \mid s,a)$ to both sides, rearranging and dividing by $\abs{\D_h^{k,b}(s,a)} + 1$ gives
\begin{align}
\label{eq:phat-freedman2}
    \brk*{1 - \frac{1}{3H}}P_h(s' \mid s,a) - \hat{P}_h^{k,b}(s' \mid s,a)
    \le
    % \frac{1}{3H} P_h(s' \mid s,a) +
    \frac{3(e-2)H \log e \delta^{-1}}{\abs{\D_h^{k,b}(s,a)} + 1}
    \le
    \frac{3(e-2) H \nb \log e \delta^{-1}}{\max\brk[c]{1, N_h^k(s,a)}}
    ,
\end{align}
where the last inequality used \cref{eq:counts-bound}.
Setting $\delta$ in \cref{eq:phat-freedman1,eq:phat-freedman2} appropriately and performing a union bound over $s', s \in \X, a \in \A, h \in [H], b \in [\nb], \abs{\D_h^{k,b}(s,a)} \in [K/\nb]$ concludes the proof.

\textbf{(\cref{eq:good-beta}):}
First, applying \cref{lemma:multiplicative-concentration} with $X_t = \sum_{h \in [H]} \beta_h^k(s_h,a_h) \le H \nb$ gives that with probability at least $1-\delta/5$
\begin{align*}
    \sum_{k \in [K]} \EE[\pi^k, P] \sum_{h \in [H]} \beta_h^k(s_h,a_h)
    \le
    2 \sum_{k \in [K]} \sum_{h \in [H]} \beta_h^k(s_h^k, a_h^k)
    +
    4 H \nb \log 5 \delta^{-1}
    .
\end{align*}
The proof is concluded by applying \cref{lemma:counts-sum}, a standard algebraic argument on the harmonic sum of the state-action counts, together with the observation that the regret is non-zero only if $A \ge 2$.
\hfill$\blacksquare$

\bibliography{bibliography}
\bibliographystyle{abbrvnat}

%%%%%%%%%%%%%%%%%%%%%%%%%%%%%%%%%%%%%%%%%%%%%%%%%%%%%%%%%%%%

% \newpage
\appendix

\section{Additional Proofs}
\subsection{Proof of Lemma \ref{lemma:quantile-of-means} (QoM)}
\label{sec:proof-of-qom}
% \begin{proof}[of \cref{lemma:quantile-of-means}]
    For the first part, notice that the event holds when at least $\ceil{\alpha \nb}$ batches yield an estimate less than or equal to $\mu$.
    By \cref{corollary:feige}, each batch satisfies this with probability at least $1/13$.
    Because the subsets $\D^b$ are fixed and disjoint, they are independent. Letting $\Bin(n,p)$ be the binomial distribution with $n$ trails and success probability $p$, we conclude that 
    \begin{align*}
        \Pr
        % &
        \brk*{
        \hat{\mu}_\alpha
        \le
        \mu
        }
        \ge
        \Pr\brk*{
        \Bin(\nb, 1/13) \ge \ceil{\nb / 65}
        }
        &
        \ge
        \Pr\brk*{
        \Bin(\nb, 1/13) > \nb / 65
        }
        \\
        &
        =
        1
        -
        \Pr\brk*{
        \Bin(\nb, 1/13) \le \nb / 65
        }
        \\
        &
        \ge
        1 - \exp(-\nb \cdot \mathrm{D_{KL}}(1/65 || 1/13))
        \\
        &
        \ge
        1 - \exp(- \nb / 26)
        \\
        &
        \ge
        1 - \delta
        ,
    \end{align*}
    where the fourth transition used a standard Chernoff bound for Binomial random variables, and $\mathrm{D_{KL}}$ is the Kullback–Leibler divergence.

    The second part of the proof follows similar logic. Denoting the mean estimate of batch $b$ in \cref{eq:quantile-of-means} as $\hat{\mu}^b$, we use \cref{lemma:freedman} (Freedman's inequality) to get
    \begin{align*}
        \Pr\brk*{
        \sum_{X \in \D^b} X 
        \ge
        \abs{\D^b} \mu
        -
        1.7\sqrt{\abs{\D^b} \sigma^2}
        -
        8 R
        }
        \ge
        0.99953
        .
    \end{align*}
    Rearranging the inner term, dividing by $\abs{\D^b} + 1 \ge \nb^{-1}\max\brk[c]{1, n}$ we get
    \begin{align*}
        \Pr\brk*{
        \hat{\mu}^b
        \ge
        \mu
        -
        1.7\sqrt{\frac{\sigma^2 \nb}{\max\brk[c]{1,n}}}
        -
        \frac{9 R \nb}{\max\brk[c]{1,n}}
        }
        \ge
        0.99953
        .
    \end{align*}
    For the event to hold, at least $\nb + 1 - \ceil{\nb / 65}$ batches need to satisfy the above. As before, this occurs with probability at least
    \begin{align*}
        \Pr\brk*{
        \Bin(\nb, 0.99953) \ge B+1-\ceil{\nb / 65}
        }
        &
        \ge
        \Pr\brk*{
        \Bin(\nb, 0.99953) > 64 \nb / 65
        }
        \\
        &
        =
        1
        -
        \Pr\brk*{
        \Bin(\nb, 0.99953) \le 64\nb / 65
        }
        \\
        &
        \ge
        1 - \exp(-\nb \cdot \mathrm{D_{KL}}(64/65 || 0.99953))
        \\
        &
        \ge
        1 - \exp(- \nb / 26)
        \\
        &
        \ge
        1 - \delta
        .
        % \qedhere
        \tag{$\blacksquare$}
    \end{align*}
% \end{proof}

\subsection{Existing Results}
\label{app:existing-results}
\subsubsection{Summing the counts}
\begin{lemma}
\label{lemma:counts-sum}
    Recall the state-action counts $N_h^k(s,a) = \sum_{k' \in [k-1]} \indEvent{(s,a) = (s_h^{k'}, a_h^{k'})}$. We have that
    \begin{align*}
        \sum_{k \in [K]} \sum_{h \in [H]} \frac{1}{\max\brk[c]{1, N_h^k(s_h^k, a_h^k)}}
        \le
        S A H \log 8 K
        .
    \end{align*}
\end{lemma}
\begin{proof}[(for sake of completeness)]
    Recall that $N_h^k(s, a) = \sum_{k' \in [k-1]} \indEvent{(s_h^{k'}, a_h^{k'}) = (s, a)}$, i.e., the counts are incremented every time a state-action pair is visited. Thus, we have
    \begin{align*}
        \sum_{k \in [K]} \sum_{h \in [H]} \frac{1}{\max\brk[c]{1, N_h^k(s_h^k, a_h^k)}}
        &
        =
        \sum_{h \in [H]} \sum_{s \in \X} \sum_{a \in \A} \sum_{k \in [K]} \frac{\indEvent{(s_h^{k}, a_h^{k}) = (s, a)}}{\max\brk[c]{1, N_h^k(s_h^k, a_h^k)}}
        \\
        &
        =
        \sum_{h \in [H]} \sum_{s \in \X} \sum_{a \in \A} \sum_{i = 0}^{N_h^K(s,a)} \frac{1}{\max\brk[c]{1, i}}
        \\
        &
        \le
        \sum_{h \in [H]} \sum_{s \in \X} \sum_{a \in \A} \sum_{i = 0}^{K} \frac{1}{\max\brk[c]{1, i}}
        \\
        &
        \le
        \sum_{h \in [H]} \sum_{s \in \X} \sum_{a \in \A} \log 8 K
        \\&
        =
        S A H \log 8 K
        ,
    \end{align*}
    where the first inequality used that $N_h^k(s,a) \le K$ for all $k \in [K]$.
\end{proof}

\subsubsection{Variance Manipulation}
The following result is a variant of Lemma 21 in \cite{merlisLookahead}.
\begin{lemma}
\label{lemma:var-star-to-pi}
    Let $V, \overline{V}, \underline{V} \in [0,H]^\X$ be value functions such that $\underline{V}(s) \le V(s) \le \overline{V}(s)$ for all $s \in \X$. Then for any $h \in [H], s \in \X, a \in \A, \beta, \eta > 0$
    % \begin{align*}
    %     \sqrt{\beta \Var_{s_h,a_h}(V(s'))}
    %     \le
    %     \sqrt{\beta \Var_{s_h,a_h}(\overline{V}(s'))}
    %     +
    %     \frac{1}{12H} P_h (\overline{{V}} - \underline{V})(s_h,a_h)
    %     +
    %     3 \beta H^2
    %     .
    % \end{align*}
    \begin{align*}
        \sqrt{\beta \Var_{s,a,h}(L_h + V(s_{h+1}))}
        \le
        \eta^{-1}
        \Var_{s,a,h}(L_h + \overline{V}(s_{h+1}))
        +
        \frac{1}{3H} P_h (\overline{{V}} - \underline{V})(s_h,a_h)
        +
        \frac{(3 H^2 + \eta) \beta}{4}
        .
    \end{align*}
\end{lemma}

\begin{proof}
    Recall that the standard deviation satisfies a triangle inequality (see \cite{zanette19a} Eqs. 48-51), i.e., for $\sigma(X) = \sqrt{\Var(X)}$ we have $\sigma(X+Y) \le \sigma(X) + \sigma(Y)$. Thus, we have
    \begin{align*}
        \sqrt{\beta \Var_{s,a,h}(L_h + V(s_{h+1}))}
        \le
        \sqrt{\beta \Var_{s,a,h}(L_h + \overline{V}(s_{h+1}))}
        +
        \sqrt{\beta \Var_{s,a,h}((V - \overline{V})(s_{h+1}))}
        .
    \end{align*}
    Now, using our assumption on the value functions, we have
    \begin{align*}
        \sqrt{\beta \Var_{s,a,h}((V - \overline{V})(s_{h+1}))}
        \le
        \sqrt{ \beta H P_h (\overline{V} - V)(s,a)}
        &
        \le
        \sqrt{ \beta H P_h (\overline{V} - \underline{V})(s,a)}
        \\
        &
        \le
        (1/3H) P_h (\overline{V} - \underline{V})(s,a)
        +
        (3/4) H^2 \beta
        ,
    \end{align*}
    where the last step also used the AM-GM inequality ($\sqrt{ab} \le (a^2 + b^2)/2$). Plugging this into the previous inequality and using AM-GM on the first term concludes the proof.
\end{proof}

Next, we state the well-known Law of Total Variance (LTV), originally due to \cite{azar2017minimax}.
\begin{lemma}[\cite{zanette19a}, Lemma 15]
\label{lemma:total-variance}
    For any policy $\pi$, it holds that
    \begin{align*}
        \EE[P,\pi] \brk[s]*{
        \sum_{h \in [H]} \Var_{s_h, a_h, h}(V^{\pi}(s_{h+1}))
        }
        =
        \EE[P,\pi]\brk[s]*{
        \sum_{h \in [H]} \ell_h(s_h, a_h) - V_1^\pi(s_1)
        }^2
    \end{align*}
\end{lemma}

\begin{lemma}[Variance Bound]
\label{lemma:variance-bound}
    For any policy $\pi: \X \to \A$ we have
    \begin{align*}
        \EE[P, \pi]
        % \brk[s]*{
        \sum_{h \in [H]} 
        \Var_{s_h, a_h, h}(L_h + V_{h+1}^{\pi}(s_{h+1}))
        % }
        =
        \EE[P,\pi, L]\brk[s]*{
        \sum_{h \in [H]} L_h - V_1^\pi(s_1)
        }^2
        \le
        H V_1^\pi(s_1)
        % \EE[P,\pi, L]\brk[s]*{
        % \sum_{h \in [H]} L_h
        % }^2
        .
    \end{align*}
\end{lemma}

\begin{proof}
    Recall the bounded total loss assumption, $0 \le \sum_{h \in [H]} L_h \le H$. Thus, we have
    \begin{align*}
        \EE[P,\pi, L]\brk[s]*{
        \sum_{h \in [H]} L_h - V_1^\pi(s_1)
        }^2
        \le
        \EE[P,\pi, L]\brk[s]*{
        \sum_{h \in [H]} L_h}^2
        \le
        H \EE[P,\pi, L]\brk[s]*{
        \sum_{h \in [H]} L_h}
        =
        H V^\pi_1(s_1)
        .
    \end{align*}
    Next, notice that $s_{h+1}$ and $L_h$ are conditionally independent on $s_h, a_h$, thus
    \begin{align*}
        \EE[P, \pi]
        \sum_{h \in [H]}
        \Var_{s_h, a_h, h}(L_h + &V_{h+1}^{\pi}(s_{h+1}))
        =
        \EE[P, \pi]
        \sum_{h \in [H]} 
        \Var_{s_h, a_h, h}(L_h)
        +
        \Var_{s_h, a_h, h}(V_{h+1}^{\pi}(s_{h+1}))
        \\
        &
        =
        \EE[P, \pi]
        \brk*{
        \sum_{h \in [H]} 
        \Var_{s_h, a_h, h}(L_h)
        }
        +
        \brk*{
        \sum_{h \in [H]} \ell_h(s_h, a_h) - V_1^\pi(s_1)
        }^2
        ,
    \end{align*}
    where the second equality used \cref{lemma:total-variance} (LTV).
    Now, let $\iota = ((s_h, a_h)_{h \in [H]}$ be a trajectory. Because $\pi$ does not depend on the losses $L_h, h \in [H]$ we have
    $
    \Var_{s_h, a_h, h}(L_h)
    =
    \EE[L] \brk[s]{
    (L_h - \ell_h(s_h,a_h))^2
    \mid \iota
    }
    $
    .
    Furthermore, conditioned on $\iota$, the losses are independent, thus
    $
    \EE[P, \pi]
    \sum_{h \in [H]} 
    \Var_{s_h, a_h, h}(L_h)
    =
    \EE[P, \pi, L]\brk*{\sum_{h \in [H]} L_h - \ell_h(s_h,a_h)}^2
    $
    .
    Plugging into the above and showing that the two terms are uncorrelated concludes the proof. To see the latter, notice that
    \begin{align*}
        \EE[P,\pi,L] &
        \brk*{\sum_{h \in [H]} L_h - \ell_h(s_h,a_h)}
        \brk*{\sum_{h \in [H]} \ell_h(s_h, a_h) - V_1^\pi(s_1)}
        \\
        &
        =
        \EE[\iota \sim P,\pi] 
        \EE[L]\brk[s]*{
        \brk*{\sum_{h \in [H]} L_h - \ell_h(s_h,a_h)}
        \Bigg| \iota
        }
        \brk*{\sum_{h \in [H]} \ell_h(s_h, a_h) - V_1^\pi(s_1)}
        =
        0
        .
        \qedhere
    \end{align*}
\end{proof}

\subsubsection{Statistical Inequalities}

% \subsubsection{Semi-Symmetry}

\begin{lemma}[\cite{feige2004sums}, Theorem 1]
\label{lemma:feige}
    Let $X_1, \ldots, X_n$ be arbitrary nonnegative independent random variables, with expectations $\mu_1, \ldots, \mu_n$ respectively, where $\mu_i \le 1$ for every $i$. Let $X = \sum_{i=1}^{n} X_i$, and let $\mu$ denote the expectation of $X$ (hence, $\mu = \sum_{i=1}^{n} \mu_i$).
    Then for every $\delta > 0$
    \begin{align*}
        \Pr\brk*{X < \mu + \delta}
        \ge
        \min\brk[c]*{
            \delta / (1+\delta), 1/13
        }
        .
    \end{align*}
\end{lemma}

% \subsubsection{Concentration Inequalities}

The following is a Freedman-type inequality \cite{freedman1975} derived by \cite{beygelzimer2011contextual}.

\begin{lemma}[\cite{beygelzimer2011contextual}, Theorem 1]
\label{lemma:freedman}
Let $X_1, X_2, \ldots, X_T$ be a sequence of real-valued random variables. Assume for all $t \in [T]$, $X_t \le R$ and $\EE\brk[s]{X_t \mid X_1, \ldots, X_{t-1}} = 0$.
Then for any $\delta \in (0,1)$ and $\lambda \in [0, 1/R]$, with probability at least $1 - \delta$
\begin{align*}
    \sum_{t \in [T]} X_t
    &
    \le
    (e-2)\lambda \sum_{t \in [T]} \EE\brk[s]{X_t^2 \mid X_1, \ldots, X_{t-1}}
    +
    \lambda^{-1} \log \frac{1}{\delta}
    % ,
    % \\
    % \sum_{t \in [T]} X_t
    % &
    % \le
    % 2 \sqrt{\sum_{t \in [T]} \EE\brk[s]{X_t^2 \mid X_1, \ldots, X_{t-1}} \log \frac{1}{\delta}}
    % +
    % R \log \frac{1}{\delta}
    .
\end{align*}
If $\EE\brk[s]{X_t^2 \mid X_1, \ldots, X_{t-1}} = \sigma^2$ then, setting $\lambda = \min\brk[c]{R^{-1}, \sqrt{((e-2)\sigma^2 T)^{-1}\log \delta^{-1}}}$, gives
\begin{align*}
    \sum_{t \in [T]} X_t
    &
    \le
    2\sqrt{(e-2)\sigma^2 T}
    +
    R \log \frac{1}{\delta}
\end{align*}
\end{lemma}

Next, we state the following Bernstein-type tail bound (e.g.,\cite{rosenberg2020near}, Lemma D.4).
\begin{lemma}
\label{lemma:multiplicative-concentration}
Let $\brk[c]{X_t}_{t \ge 1}$
be a sequence of random variables with expectation adapted to a filtration
$\mathcal{F}_t$.
Suppose that $0 \le X_t \le 1$ almost surely. Then with probability at least $1-\delta$
\begin{align*}
    \sum_{t=1}^{T} \EE \brk[s]{X_t \mid \mathcal{F}_{t-1}}
    \le
    2 \sum_{t=1}^{T} X_t
    +
    4 \log \frac{1}{\delta}
\end{align*}
\end{lemma}

\section{Extension to Heavy-Tailed Losses}
\label{sec:heavy-tail-extension}

  In this section we discuss how \cref{thm:vibe} extends to losses with heavy tails. Specifically, we can replace Assumption~\ref{assumption:bounded-total-loss} with the weaker requirement that the
  random losses $L_h$ are non-negative, satisfy $\EE\brk[s]*{\sum_{h \in [H]} L_h} \le H$ for all policies, and each $L_h(s,a)$ has finite (but unknown) variance.

  The key observation enabling this extension is that \cref{lemma:quantile-of-means} applies to non-negative random variables with finite variance, without requiring boundedness (as discussed
  following its statement). Consequently, the QoM estimator remains optimistic and its bias increases by at most a constant multiplicative factor in this setting. Given this, the good event
  used in the proof of \cref{thm:vibe} holds essentially unchanged, and the remainder of the regret decomposition follows straightforwardly, yielding the same regret bound up to constants.

  We note, however, that the $\min\brk[c]{\QQ, H V^\star}$ term appearing in \cref{thm:vibe} reduces to $\QQ$ alone in the heavy-tailed setting. The $H V^\star$ term arises from bounding the
  variance of the value function under the optimal policy, which relies on the losses being almost surely bounded. Since no such bound is assumed in the heavy-tailed setting, this refinement
  does not apply, and we cannot expect the $H V^\star$ term to appear in the regret bound.

\section{Multi-Armed Bandits}
\label{sec:MAB}
\paragraph{Problem setup.}
A stochastic multi-armed bandit is an MDP with a single state and horizon one, i.e., $S=H=1$ (see \cref{sec:setting}). Thus, we simplify the notation as follows. In each episode, $k \in [K]$ the agent chooses an arm $a^k \in \A$ and observes a stochastic loss $L^k \in [0,1]$ sampled independently from a distribution $\LL(a^k)$ whose mean and variance are denoted $\ell(a^k), \sigma^2(a^k)$.
Denoting the optimal action and value as $\aOpt \in \argmin_{a \in \A} \ell(a)$ and $\ellOpt = \mu(\aOpt)$, the (pseudo) regret can be rewritten as
\begin{align*}
    \regret
    =
    \sum_{k \in [K]} [\ell(a^k) - \ellOpt]
    =
    \sum_{k \in [K]} \Delta(a^k)
    ,
\end{align*}
where $\Delta(a) = \ell(a) - \ellOpt$ is known as the sub-optimality gap of arm $a \in \A$.

\subsection{Algorithm and Main Results}
We present \cref{alg:qom-mab}, which is identical to \cref{alg:vibe} but with a simplified notation.

\begin{algorithm}[!ht]
	\caption{Quantile of Mean for MAB} \label{alg:qom-mab}
	\begin{algorithmic}[1]
	    \State \textbf{input}:
        Number of batches $\nb$, quantile level $\alpha$.

        \State \textbf{initialize}: $N^1(a) = 0$, $\D^{1,b}(a) = \emptyset$ for all $a \in A$.
        
	    \For{episode $k = 1,2,\ldots, K$}
        
            \begin{alignat*}{2}
                &
                \hat{\ell}^{k,b}(a)
                &&
                =
                \sum_{L \in \D^{k,b}(a)} \frac{L}{\abs{\D^{k,b}(a)} + 1}
                \\
                &
                \hat{\ell}^{k}(a)
                &&
                =
                \quantile\brk*{
                \hat{\ell}^{k,b}(a)
                ,
                b \in [\nb]
                }
                .
            \end{alignat*}
            % \label{line:vibe-tabular}

            \State Play $a^k \in \argmin_{a \in \A} \hat{\ell}^k(a)$ and observe loss $L^k$.

            \State Update global counts: $N^{k+1}(a) = N^{k}(a) + \indEvent{a = a^k}$

            \State Calculate batch index: $b^k = 1 + (N^k(a^k) \mod{B})$
            and update datasets:
            \begin{align*}
                \D^{k+1,b}(a) = \D^{k,b}(a) \bigcup
                \begin{cases}
                    L^k, &\text{if } (a,b) = (a^k, b^k)
                    \\
                    \emptyset, &\text{otherwise}.
                \end{cases}
            \end{align*}
            
	    \EndFor
	\end{algorithmic}
	
\end{algorithm}
The following is our main result for multi-armed bandits (proof in \cref{sec:qom-mab-proofs}).
\begin{theorem}
\label{thm:qom-mab}
    Suppose we run \cref{alg:qom-mab} with number of batches $\nb = 26 \log (2 K A \delta^{-1})$, and quantile level $\alpha = 1/65$. With probability at least $1-\delta$, the following hold simultaneously
    \begingroup\allowdisplaybreaks
    \begin{align*}
        % &
        \regret
        &
        \le
        \sqrt{76 K \sum_{a \neq \aOpt}\sigma^2(a) \log (6 K A \delta^{-1})
        }
        +
        469 A \log (2 K A \delta^{-1})
        \\
        % &
        \regret
        &
        \le
        \sum_{a \neq \aOpt} \brk*{
            \frac{76 \sigma^2(a)}{\Delta(a)}
            +
            469
        }\log (2 K A \delta^{-1})
        .
    \end{align*}
    \endgroup
\end{theorem}
We note that \cite{feige2004sums} conjecture that their result, stated here as \cref{lemma:feige}, holds with a constant of $1/e$. If true, this would change the numerical constant in the batch size to $5.5$, and a better choice of $\alpha$ would be $1/10$. Overall, this would decrease the constants in our bounds.
Finally, we note that this result does not require the losses to be bounded in $[0,1]$. In fact, it holds even for nonnegative losses with finite variance, albeit with significantly larger numerical constants. To see this, replace the use of Freedman's inequality in the proof of \cref{lemma:quantile-of-means} with Chebyshev's inequality.

\subsection{Proof of Theorem \ref{thm:qom-mab}}
\label{sec:qom-mab-proofs}
% \begin{proof}
Recall that the pseudo-regret may be written as
\begin{align}
\label{eq:mab-reg-decomp}
    \regret
    =
    \sum_{a \neq \aOpt} N^{K+1}(a) \Delta(a)
    ,
\end{align}
where $N^k(a)$ is defined in \cref{alg:qom-mab}. Thus, our goal is to bound $N^{K+1}(a)$ for each sub-optimal arm.
We begin with a standard ``good event'' over which the regret is bounded deterministically. Suppose that for all $k \in [K]$ and $a \neq \aOpt$ we have
\begin{align}
    &
    \label{eq:good-optimism-mab}
    \hat{\ell}^k(\aOpt) \le \ellOpt
    \\
    &
    \label{eq:good-concentration-mab}
    \hat{\ell}^k(a)
    \ge
    \ell(a)
    -
    \frac{9 \nb}{N^k(a)}
    -
    1.7 \sqrt{\frac{\sigma^2(a) \nb}{N^k(a)}}
    .
\end{align}
Taking a union bound over \cref{lemma:quantile-of-means} with $\delta/2KA$, the above holds with probability at least $1 - \delta$.
Now, suppose that arm $a$ with $\Delta(a) > 0$ was played at episode $k$. Then, by the decision rule in \cref{alg:qom-mab}, we have
$
    \hat{\ell}^k(\aOpt) \ge \hat{\ell}^k(a)
    ,
$
and thus
\begin{align*}
    \ellOpt
    \tag{\cref{eq:good-optimism-mab}}
    &
    \ge
    \hat{\ell}^k(\aOpt)
    \\
    &
    \ge
    \hat{\ell}^k(a)
    \\
    &
    \tag{\cref{eq:good-concentration-mab}}
    \ge
    \ell(a)
    -
    \frac{9 \nb}{N^k(a)}
    -
    1.7\sqrt{\frac{\sigma^2(a) \nb}{N^k(a)}}
    .
\end{align*}
Solving this quadratic inequality for $N^k(a)$, we have
\begin{align*}
    N^k(a)
    \le
    \nb \brk[s]*{
        \frac{2.9 \sigma^2(a)}{\Delta^2(a)}
        +
        \frac{18}{\Delta(a)}
    }
    ,
\end{align*}
Now, let $k_a$ be the last episode when arm $a$ was chosen. Then we have
\begin{align*}
    N^{K+1}(a)
    =
    N^{k_a}(a) + 1
    &
    \le
    \brk*{
        \frac{76 \sigma^2(a)}{\Delta^2(a)}
        +
        \frac{469}{\Delta(a)}
    }\log (2 K A \delta^{-1})
    ,
\end{align*}
where the inequality also used our choice of $\nb = 26 \log(2 K A \delta^{-1})$.
Plugging this into \cref{eq:mab-reg-decomp} concludes the instance-dependent regret bounds.

Next, to obtain the instance-independent bounds, notice that
\begin{align*}
    N^{K+1}(a)\Delta(a)
    &
    \le
    \min \brk[c]*{
        N^{K+1}(a)\Delta(a)
        ,
        \frac{76 \sigma^2(a)}{\Delta(a)}\log (2 K A \delta^{-1})
    }
    +
    469 \log (2 K A \delta^{-1})
    \\
    &
    \le
    \sqrt{
        76 N^{K+1}(a) \sigma^2(a) \log (2 K A \delta^{-1})
    }
    +
    469 \log (2 K A \delta^{-1})
    ,
\end{align*}
where the second inequality used $\min\brk[c]{c_1, c_2} \le \sqrt{c_1 c_2}$. Plugging this into \cref{eq:mab-reg-decomp} and applying the Cauchy-Schwarz inequality concludes the proof.
\hfill$\blacksquare$

% \newpage
% \input{checklist.tex}

\end{document}

%% file: Defs/thmdefs.tex
\usepackage{amsthm}
\usepackage{amssymb}
\usepackage{thmtools}

\declaretheoremstyle[
	    spaceabove=\topsep, 
	    spacebelow=\topsep, 
	    headfont=\normalfont\bfseries,
	    bodyfont=\normalfont\itshape,
	    notefont=\normalfont\bfseries,
	    notebraces={(}{)},
	    postheadspace=0.33em, 
	    headpunct={.},
    ]{theorem}
\declaretheorem[style=theorem]{theorem}

\declaretheoremstyle[
	    spaceabove=\topsep, 
	    spacebelow=\topsep, 
	    headfont=\normalfont\bfseries,
	    bodyfont=\normalfont,
	    notefont=\normalfont\bfseries,
	    notebraces={(}{)},
	    postheadspace=0.33em, 
	    headpunct={.},
    ]{definition}

\declaretheoremstyle[
        spaceabove=\topsep, 
        spacebelow=\topsep, 
        headfont=\normalfont\bfseries,
        bodyfont=\normalfont,
        notefont=\normalfont\bfseries,
        notebraces={}{},
        postheadspace=0.33em, 
        qed=$\blacksquare$, 
        headpunct={.},
    ]{proofstyle}
\declaretheorem[style=proofstyle,numbered=no,name=Proof]{proof}

\declaretheorem[style=theorem,sibling=theorem,name=Lemma]{lemma}
\declaretheorem[style=theorem,sibling=theorem,name=Corollary]{corollary}

\declaretheorem[style=theorem,sibling=theorem,name=Assumption]{assumption}

\declaretheorem[style=theorem,numbered=no,name=Theorem]{theorem*}
\declaretheorem[style=theorem,numbered=no,name=Lemma]{lemma*}
\declaretheorem[style=theorem,numbered=no,name=Corollary]{corollary*}
\declaretheorem[style=theorem,numbered=no,name=Proposition]{proposition*}
\declaretheorem[style=theorem,numbered=no,name=Claim]{claim*}
\declaretheorem[style=theorem,numbered=no,name=Fact]{fact*}
\declaretheorem[style=theorem,numbered=no,name=Observation]{observation*}
\declaretheorem[style=theorem,numbered=no,name=Conjecture]{conjecture*}

\declaretheorem[style=definition,numbered=no,name=Definition]{definition*}
\declaretheorem[style=definition,numbered=no,name=Remark]{remark*}
\declaretheorem[style=definition,numbered=no,name=Example]{example*}
\declaretheorem[style=definition,numbered=no,name=Question]{question*}

%% file: Defs/macros.tex
\newcommand{\D}{\mathcal{D}}
\newcommand{\A}{\mathcal{A}}
\newcommand{\X}{\mathcal{S}}

\newcommand{\M}{\mathcal{M}}
\newcommand{\LL}{\mathcal{L}}
\newcommand{\piDMClass}{\Pi_{M}}

\newcommand{\quantile}[1][\alpha]{q_{#1}}

\newcommand{\regret}[1][K]{\mathrm{regret}_{#1}}

\newcommand{\nb}{B}

\newcommand{\goodEvent}{\mathcal{E}}

\newcommand{\QQ}{\mathbb{Q}^\star}
\newcommand{\QQh}[1][h]{\mathbb{Q}_{#1}^\star}

\newcommand{\Bin}{\mathrm{Bin}}

\newcommand{\aOpt}{a^\star}
\newcommand{\ellOpt}{\ell^\star}

%% file: bibliography.bib
@InProceedings{kaufmann2012bayesian,
  title = 	 {On Bayesian Upper Confidence Bounds for Bandit Problems},
  author = 	 {Kaufmann, Emilie and Cappe, Olivier and Garivier, Aurelien},
  booktitle = 	 {Proceedings of the Fifteenth International Conference on Artificial Intelligence and Statistics},
  pages = 	 {592--600},
  year = 	 {2012},
  editor = 	 {Lawrence, Neil D. and Girolami, Mark},
  volume = 	 {22},
  series = 	 {Proceedings of Machine Learning Research},
  address = 	 {La Palma, Canary Islands},
  month = 	 {21--23 Apr},
  publisher =    {PMLR}
}

@inproceedings{hao2019bootstrapping,
 author = {Hao, Botao and Abbasi Yadkori, Yasin and Wen, Zheng and Cheng, Guang},
 booktitle = {Advances in Neural Information Processing Systems},
 editor = {H. Wallach and H. Larochelle and A. Beygelzimer and F. d\textquotesingle Alch\'{e}-Buc and E. Fox and R. Garnett},
 pages = {},
 publisher = {Curran Associates, Inc.},
 title = {Bootstrapping Upper Confidence Bound},
 volume = {32},
 year = {2019}
}

@inproceedings{lee2024improved,
 author = {Lee, Harin and Oh, Min-hwan},
 booktitle = {Advances in Neural Information Processing Systems},
 doi = {10.52202/079017-2947},
 editor = {A. Globerson and L. Mackey and D. Belgrave and A. Fan and U. Paquet and J. Tomczak and C. Zhang},
 pages = {92803--92831},
 publisher = {Curran Associates, Inc.},
 title = {Improved Regret of Linear Ensemble Sampling},
 volume = {37},
 year = {2024}
}

@inproceedings{tiapkin2022opsrl,
 author = {Tiapkin, Daniil and Belomestny, Denis and Calandriello, Daniele and Moulines, Eric and Munos, Remi and Naumov, Alexey and Rowland, Mark and Valko, Michal and M\'{e}nard, Pierre},
 booktitle = {Advances in Neural Information Processing Systems},
 editor = {S. Koyejo and S. Mohamed and A. Agarwal and D. Belgrave and K. Cho and A. Oh},
 pages = {10737--10751},
 publisher = {Curran Associates, Inc.},
 title = {Optimistic Posterior Sampling for Reinforcement Learning with Few Samples and Tight Guarantees},
 volume = {35},
 year = {2022}
}

@InProceedings{tiapkin2022dirichlet,
  title = 	 {From {D}irichlet to Rubin: Optimistic Exploration in {RL} without Bonuses},
  author =       {Tiapkin, Daniil and Belomestny, Denis and Moulines, Eric and Naumov, Alexey and Samsonov, Sergey and Tang, Yunhao and Valko, Michal and Menard, Pierre},
  booktitle = 	 {Proceedings of the 39th International Conference on Machine Learning},
  pages = 	 {21380--21431},
  year = 	 {2022},
  editor = 	 {Chaudhuri, Kamalika and Jegelka, Stefanie and Song, Le and Szepesvari, Csaba and Niu, Gang and Sabato, Sivan},
  volume = 	 {162},
  series = 	 {Proceedings of Machine Learning Research},
  month = 	 {17--23 Jul},
  publisher =    {PMLR}
}

@InProceedings{viel2025ilsoar,
  title = 	 {{IL}-{SOAR} : Imitation Learning with Soft Optimistic Actor c{R}itic},
  author =       {Viel, Stefano and Viano, Luca and Cevher, Volkan},
  booktitle = 	 {Proceedings of the 42nd International Conference on Machine Learning},
  pages = 	 {61444--61479},
  year = 	 {2025},
  editor = 	 {Singh, Aarti and Fazel, Maryam and Hsu, Daniel and Lacoste-Julien, Simon and Berkenkamp, Felix and Maharaj, Tegan and Wagstaff, Kiri and Zhu, Jerry},
  volume = 	 {267},
  series = 	 {Proceedings of Machine Learning Research},
  month = 	 {13--19 Jul},
  publisher =    {PMLR}
}

@article{jin2018is,
  title={Is Q-learning provably efficient?},
  author={Jin, Chi and Allen-Zhu, Zeyuan and Bubeck, Sebastien and Jordan, Michael I},
  journal={Advances in neural information processing systems},
  volume={31},
  year={2018}
}

@article{agrawal2017optimistic,
  title={Optimistic posterior sampling for reinforcement learning: worst-case regret bounds},
  author={Agrawal, Shipra and Jia, Randy},
  journal={Advances in neural information processing systems},
  volume={30},
  year={2017}
}

@article{russo2019worst,
  title={Worst-case regret bounds for exploration via randomized value functions},
  author={Russo, Daniel},
  journal={Advances in neural information processing systems},
  volume={32},
  year={2019}
}

@article{dann2017unifying,
  title={Unifying PAC and regret: Uniform PAC bounds for episodic reinforcement learning},
  author={Dann, Christoph and Lattimore, Tor and Brunskill, Emma},
  journal={Advances in Neural Information Processing Systems},
  volume={30},
  year={2017}
}

@article{lu2017ensemble,
  title={Ensemble sampling},
  author={Lu, Xiuyuan and Van Roy, Benjamin},
  journal={Advances in neural information processing systems},
  volume={30},
  year={2017}
}

@InProceedings{domingues21a,
  title = 	 {Episodic Reinforcement Learning in Finite MDPs: Minimax Lower Bounds Revisited},
  author =       {Domingues, Omar Darwiche and M{\'e}nard, Pierre and Kaufmann, Emilie and Valko, Michal},
  booktitle = 	 {Proceedings of the 32nd International Conference on Algorithmic Learning Theory},
  pages = 	 {578--598},
  year = 	 {2021},
  editor = 	 {Feldman, Vitaly and Ligett, Katrina and Sabato, Sivan},
  volume = 	 {132},
  series = 	 {Proceedings of Machine Learning Research},
  month = 	 {16--19 Mar},
  publisher =    {PMLR},
}

@InProceedings{zhou2023sharp,
  title = 	 {Sharp Variance-Dependent Bounds in Reinforcement Learning: Best of Both Worlds in Stochastic and Deterministic Environments},
  author =       {Zhou, Runlong and Zihan, Zhang and Du, Simon Shaolei},
  booktitle = 	 {Proceedings of the 40th International Conference on Machine Learning},
  pages = 	 {42878--42914},
  year = 	 {2023},
  editor = 	 {Krause, Andreas and Brunskill, Emma and Cho, Kyunghyun and Engelhardt, Barbara and Sabato, Sivan and Scarlett, Jonathan},
  volume = 	 {202},
  series = 	 {Proceedings of Machine Learning Research},
  month = 	 {23--29 Jul},
  publisher =    {PMLR}
}

@article{mnih2015human,
  title={Human-level control through deep reinforcement learning},
  author={Mnih, Volodymyr and Kavukcuoglu, Koray and Silver, David and Rusu, Andrei A and Veness, Joel and Bellemare, Marc G and Graves, Alex and Riedmiller, Martin and Fidjeland, Andreas K and Ostrovski, Georg and others},
  journal={nature},
  volume={518},
  number={7540},
  pages={529--533},
  year={2015},
  publisher={Nature Publishing Group}
}

@inproceedings{schulman2015trust,
  title={Trust region policy optimization},
  author={Schulman, John and Levine, Sergey and Abbeel, Pieter and Jordan, Michael and Moritz, Philipp},
  booktitle={International conference on machine learning},
  pages={1889--1897},
  year={2015},
  organization={PMLR}
}

@article{schulman2017proximal,
  title={Proximal policy optimization algorithms},
  author={Schulman, John and Wolski, Filip and Dhariwal, Prafulla and Radford, Alec and Klimov, Oleg},
  journal={arXiv preprint arXiv:1707.06347},
  year={2017}
}

@inproceedings{haarnoja2018soft,
  title={Soft actor-critic: Off-policy maximum entropy deep reinforcement learning with a stochastic actor},
  author={Haarnoja, Tuomas and Zhou, Aurick and Abbeel, Pieter and Levine, Sergey},
  booktitle={International Conference on Machine Learning},
  pages={1861--1870},
  year={2018},
  organization={PMLR}
}

@article{stiennon2020learning,
  title={Learning to summarize with human feedback},
  author={Stiennon, Nisan and Ouyang, Long and Wu, Jeffrey and Ziegler, Daniel and Lowe, Ryan and Voss, Chelsea and Radford, Alec and Amodei, Dario and Christiano, Paul F},
  journal={Advances in Neural Information Processing Systems},
  volume={33},
  pages={3008--3021},
  year={2020}
}

@article{ouyang2022training,
  title={Training language models to follow instructions with human feedback},
  author={Ouyang, Long and Wu, Jeffrey and Jiang, Xu and Almeida, Diogo and Wainwright, Carroll and Mishkin, Pamela and Zhang, Chong and Agarwal, Sandhini and Slama, Katarina and Ray, Alex and others},
  journal={Advances in Neural Information Processing Systems},
  volume={35},
  pages={27730--27744},
  year={2022}
}

@inproceedings{cassel2025batch,
  title={Batch ensemble for variance dependent regret in stochastic bandits},
  author={Cassel, Asaf and Levy, Orin and Mansour, Yishay},
  booktitle={Proceedings of the AAAI Conference on Artificial Intelligence},
  volume={39},
  pages={15678--15685},
  year={2025}
}

@inproceedings{lee2021sunrise,
  title={Sunrise: A simple unified framework for ensemble learning in deep reinforcement learning},
  author={Lee, Kimin and Laskin, Michael and Srinivas, Aravind and Abbeel, Pieter},
  booktitle={International conference on machine learning},
  pages={6131--6141},
  year={2021},
  organization={PMLR}
}

@inproceedings{chen2021randomized,
  author       = {Xinyue Chen and
                  Che Wang and
                  Zijian Zhou and
                  Keith W. Ross},
  title        = {Randomized Ensembled Double Q-Learning: Learning Fast Without a Model},
  booktitle    = {9th International Conference on Learning Representations, {ICLR} 2021,
                  Virtual Event, Austria, May 3-7, 2021},
  publisher    = {OpenReview.net},
  year         = {2021},
}

@inproceedings{pathak2019self,
  title={Self-supervised exploration via disagreement},
  author={Pathak, Deepak and Gandhi, Dhiraj and Gupta, Abhinav},
  booktitle={International conference on machine learning},
  pages={5062--5071},
  year={2019},
  organization={PMLR}
}

@inproceedings{rosenberg2020near,
  title={Near-optimal regret bounds for stochastic shortest path},
  author={Rosenberg, Aviv and Cohen, Alon and Mansour, Yishay and Kaplan, Haim},
  booktitle={International Conference on Machine Learning},
  pages={8210--8219},
  year={2020},
  organization={PMLR}
}

@inproceedings{azar2017minimax,
  title={Minimax regret bounds for reinforcement learning},
  author={Azar, Mohammad Gheshlaghi and Osband, Ian and Munos, R{\'e}mi},
  booktitle={International conference on machine learning},
  pages={263--272},
  year={2017},
  organization={PMLR}
}

@InProceedings{zanette19a,
  title = 	 {Tighter Problem-Dependent Regret Bounds in Reinforcement Learning without Domain Knowledge using Value Function Bounds},
  author =       {Zanette, Andrea and Brunskill, Emma},
  booktitle = 	 {Proceedings of the 36th International Conference on Machine Learning},
  pages = 	 {7304--7312},
  year = 	 {2019},
  editor = 	 {Chaudhuri, Kamalika and Salakhutdinov, Ruslan},
  volume = 	 {97},
  series = 	 {Proceedings of Machine Learning Research},
  month = 	 {09--15 Jun},
  publisher =    {PMLR},
}

@inproceedings{merlisLookahead,
 author = {Merlis, Nadav},
 booktitle = {Advances in Neural Information Processing Systems},
 editor = {A. Globerson and L. Mackey and D. Belgrave and A. Fan and U. Paquet and J. Tomczak and C. Zhang},
 pages = {64523--64581},
 publisher = {Curran Associates, Inc.},
 title = {Reinforcement Learning with Lookahead Information},
 volume = {37},
 year = {2024}
}

@inproceedings{beygelzimer2011contextual,
  title={Contextual bandit algorithms with supervised learning guarantees},
  author={Beygelzimer, Alina and Langford, John and Li, Lihong and Reyzin, Lev and Schapire, Robert},
  booktitle={Proceedings of the Fourteenth International Conference on Artificial Intelligence and Statistics},
  pages={19--26},
  year={2011},
  organization={JMLR Workshop and Conference Proceedings}
}

@article{freedman1975,
 ISSN = {00911798, 2168894X},
 author = {David A. Freedman},
 journal = {The Annals of Probability},
 number = {1},
 pages = {100--118},
 publisher = {Institute of Mathematical Statistics},
 title = {On Tail Probabilities for Martingales},
 volume = {3},
 year = {1975}
}

@inproceedings{feige2004sums,
  title={On sums of independent random variables with unbounded variance, and estimating the average degree in a graph},
  author={Feige, Uriel},
  booktitle={Proceedings of the thirty-sixth annual ACM symposium on Theory of computing},
  pages={594--603},
  year={2004}
}

@article{chen2017ucb,
  title={Ucb exploration via q-ensembles},
  author={Chen, Richard Y and Sidor, Szymon and Abbeel, Pieter and Schulman, John},
  journal={arXiv preprint arXiv:1706.01502},
  year={2017}
}

@inproceedings{peer2021ensemble,
  title={Ensemble bootstrapping for Q-Learning},
  author={Peer, Oren and Tessler, Chen and Merlis, Nadav and Meir, Ron},
  booktitle={International Conference on Machine Learning},
  pages={8454--8463},
  year={2021},
  organization={PMLR}
}

@inproceedings{osband2016generalization,
  title={Generalization and exploration via randomized value functions},
  author={Osband, Ian and Van Roy, Benjamin and Wen, Zheng},
  booktitle={International Conference on Machine Learning},
  pages={2377--2386},
  year={2016},
  organization={PMLR}
}

@article{osband2013more,
  title={(More) efficient reinforcement learning via posterior sampling},
  author={Osband, Ian and Russo, Daniel and Van Roy, Benjamin},
  journal={Advances in Neural Information Processing Systems},
  volume={26},
  year={2013}
}

@article{osband2016deep,
  title={Deep exploration via bootstrapped DQN},
  author={Osband, Ian and Blundell, Charles and Pritzel, Alexander and Van Roy, Benjamin},
  journal={Advances in neural information processing systems},
  volume={29},
  year={2016}
}

@inproceedings{KvetonSVWLG19-GIRO,
  author       = {Branislav Kveton and
                  Csaba Szepesv{\'{a}}ri and
                  Sharan Vaswani and
                  Zheng Wen and
                  Tor Lattimore and
                  Mohammad Ghavamzadeh},
  editor       = {Kamalika Chaudhuri and
                  Ruslan Salakhutdinov},
  title        = {Garbage In, Reward Out: Bootstrapping Exploration in Multi-Armed Bandits},
  booktitle    = {Proceedings of the 36th International Conference on Machine Learning,
                  {ICML} 2019, 9-15 June 2019, Long Beach, California, {USA}},
  series       = {Proceedings of Machine Learning Research},
  volume       = {97},
  pages        = {3601--3610},
  publisher    = {{PMLR}},
  year         = {2019},
}
